\documentclass[letterpaper]{article} 
\usepackage{aaai25}  
\usepackage{times}  
\usepackage{helvet}  
\usepackage{courier}  
\usepackage[hyphens]{url}  
\usepackage{graphicx} 
\usepackage{natbib}  
\usepackage{caption} 
\usepackage{algorithm}
\usepackage{algorithmic}

\usepackage{newfloat}
\usepackage{listings}
\DeclareCaptionStyle{ruled}{labelfont=normalfont,labelsep=colon,strut=off} 
\floatstyle{ruled}
\newfloat{listing}{tb}{lst}{}
\floatname{listing}{Listing}
\usepackage{latexsym}
\usepackage{amssymb}
\usepackage{amsmath}
\usepackage{amsthm}
\usepackage{booktabs}
\usepackage{enumitem}
\usepackage{graphicx}
\usepackage{color}

\usepackage{thm-restate}

\newtheorem{theorem}{Theorem}[section]
\newtheorem{lemma}[theorem]{Lemma}

\newtheorem{definition}{Definition}
\newtheorem{example}{Example}

\usepackage{cleveref}
\crefname{equation}{eq.}{eqs.}
\Crefname{equation}{Eq.}{Eqs.}

\newcommand{\jiarui}[1]{{\color{blue}[JG: #1]}}
\newcommand{\jonathan}[1]{{\color{orange}[JS: #1]}}

\newcommand{\todo}[1]{{\color{red}[TO DO: #1]}}

\newcommand{\cupdot}{\mathbin{\mathaccent\cdot\cup}}

\newcommand{\given}{{\,|\,}}
\DeclareMathOperator{\Pro}{\mathbb{P}}
\DeclareMathOperator{\E}{\mathbb{E}}
\DeclareMathOperator{\supp}{supp}
\newcommand{\xpub}{\text{\textnormal{pub}}}
\newcommand{\xsm}{\text{\textnormal{sem}}}
\newcommand{\xpv}{\text{\textnormal{prv}}}
\newcommand{\sig}{\phi}
\newcommand{\sigtyped}{\phi^{\text{T}}}
\newcommand{\sigpub}{\phi^{\xpub}}
\newcommand{\sigsm}{\phi^{\xsm}}
\newcommand{\sigpv}{\phi^{\xpv}}
\newcommand{\pubblockingprofiles}{B^\xpub}
\newcommand{\smblockingprofiles}{B^\xsm}
\newcommand{\pvblockingprofiles}{B^\xpv}

\newcommand{\Worlds}{{\Omega}}
\newcommand{\priors}{{{\mu}}}
\newcommand{\Types}{{{\mathcal{T}}}}
\newcommand{\utility}{{{u}}}
\newcommand{\Maxdeviators}{{{d}}}
\newcommand{\profiles}{{{P}}}

\newcommand{\Actions}{{{A}}}
\newcommand{\Actionvectors}{{\mathcal{A}}}
\newcommand{\actionvector}{{\mathbf{a}}}
\newcommand{\signalvector}{\mathbf{g}}
\newcommand{\Signalvectors}{\mathcal{G}}
\newcommand{\Representativeactionvectors}{{\bar{\mathcal{A}}}}

\newcommand{\Agents}{{N}}

\newcommand{\toprofile}{{\boldsymbol{\varphi}}}

\newcommand{\torepresentative}{{\boldsymbol{\varphi^*}}}

\newcommand{\Deviationprofiles}{{{D}}}

\newcommand{\Usefulpublicsignals}{{\boldsymbol{{S}}}}

\newcommand{\Deviationagents}{{\boldsymbol{\overrightarrow{D}}}}
\newcommand{\Blockingagents}{{\boldsymbol{\overrightarrow{B}}}}
\newcommand{\Usefulsemisignals}{{\boldsymbol{\overrightarrow{\mathcal{S}}}}}

\newcommand{\whatever}{{\boldsymbol{a^*}}}
\newcommand{\actionsandwhatever}{{\boldsymbol{A^*}}}

\newcommand{\lotterypolicy}{\lambda}

\newcommand{\Permutations}{{\mathcal{M}}}

\newcommand{\tosignal}{{\boldsymbol{g}}}

\title{Bayesian Persuasion with Externalities: Exploiting Agent Types}

\author {
    Jonathan Shaki\textsuperscript{\rm 1},
    Jiarui Gan\textsuperscript{\rm 2},
    Sarit Kraus\textsuperscript{\rm 1}
}
\affiliations {
    \textsuperscript{\rm 1}Bar-Ilan University\\
    \textsuperscript{\rm 2}University of Oxford\\
    jonath.shaki@gmail.com, jiarui.gan@cs.ox.ac.uk, sarit@cs.biu.ac.il
}

\begin{document}

\maketitle

\begin{abstract}
We study a Bayesian persuasion problem with externalities. In this model, a principal sends signals to inform multiple agents about the state of the world. Simultaneously, due to the existence of externalities in the agents' utilities, the principal also acts as a correlation device to correlate the agents' actions. We consider the setting where the agents are categorized into a small number of types. Agents of the same type share identical utility functions and are treated equitably in the utility functions of both other agents and the principal. We study the problem of computing optimal signaling strategies for the principal, under three different types of signaling channels: public, private, and semi-private.  Our results include revelation-principle-style characterizations of optimal signaling strategies, linear programming formulations, and analysis of in/tractability of the optimization problems. It is demonstrated that when the maximum number of deviating agents is bounded by a constant, our LP-based formulations compute optimal signaling strategies in polynomial time. Otherwise, the problems are NP-hard.
\end{abstract}

\section{Introduction}

Bayesian persuasion, introduced by \cite{kamenica2011bayesian}, is a framework where a persuader (the principal) seeks to design a signaling scheme that conveys information to a decision-maker (the agent) to influence their choices in the principal's desired direction ~\cite{KamenicaSurvey, fujii2022algorithmic, celli2020private}.  
The computational analysis of multi-agent Bayesian persuasion has been studied extensively in recent years, considering various domains, such as 
congestion games \cite{vasserman2015implementing,castiglioni2021signaling,bhaskar2016hardness,azaria2014strategic,das2017reducing,griesbach2022public}, 
auctions \cite{li2019revenue,Gatti22}, security \cite{rabinovich2015information,zhou2023information}, and advertising \cite{kumar2023persuasion}.
In this paper, we consider persuading multiple agents in the presence of {\em externalities} in the agents' utilities---that is, the utility of an agent may depend on other agents' actions as well, instead of just their own---and we consider a fairly general model that is not tied to specific applications (e.g., congestion games).

Persuading multiple agents in the presence of externalities is computationally challenging for several reasons.
First, even when there are no externalities, public persuasion (i.e., when signals are sent through a public channel visible to every agent) is known to be intractable. It is shown that no polynomial time approximation scheme (PTAS) exists for computing optimal public persuasion policies unless P = NP \cite{bhaskar2016hardness,rubinstein2017honest,dughmi2019hardness}. When introducing externalities, the problem does not become computationally easier, as it is a strict generalization of the case without externalities.
For public persuasion, the above hardness result holds automatically, whereas for private persuasion, while the case without externalities can be trivially solved as independent single-agent persuasion problems against individual agents, this is not possible when externalities are present. 
In the presence of the externalities, the actions of the agents need to be coordinated, too. Indeed, coordination is necessary even when the principal does not possess any private information about the hidden state, whereby the goal of the principal is essentially to induce a correlated equilibrium among the agents that is optimal for the principal. Hence, the problem we address is essentially a combination of optimal persuasion and coordination. 

Given the computational barriers, we consider a more amenable but fairly natural special case of the model where the agents can be classified into a small number of {\em types}. 
Similarly to other works in game theory ~\cite{shrot2010agent, lovejoy2006optimal},
agents of the same type share identical utility function and are treated equitably in the utility functions of the other players (however, they still act independently and need not take the same action).
Hence, the agents' joint action can be succinctly represented as a vector that describes, for each agent type $T$ and action $a$, how many agents of type $T$ are performing action $a$.
Succinct representations are necessary for many applications as they avoid explicitly defining the utilities of possible joint actions, which grow exponentially with the number of agents.
From a computational perspective, however, this means that designing an efficient algorithm, which runs in time polynomial in the size of the problem representation, becomes harder.
Indeed, it has been shown that correlated equilibria are hard to optimize in many succinctly represented games \cite{papadimitriou2005computing}. 


\subsection{Contributions and Technical Novelties} 
We make the following contributions in this paper.

\begin{itemize}

\item We show that the classical revelation principle does not hold if agents can deviate jointly. We provide alternative revelation-principle-style characterizations for our model.
Speaking informally, the new characterizations show that the principal can restrict herself to sending the agents the actions they are recommended to play, along with a blocking profile, which encodes necessary information that explains, for every possible deviation, why the deviating agents will not benefit from this deviation.

\item Based on the characterizations, we present polynomial-time algorithms to compute optimal policies of the principal when joint deviations of the agents are restricted among a constant number of agents. In the case without this restriction, we show that computing optimal policies is intractable.

\item We consider both public and private persuasion as the two most commonly studied forms of persuasion in the literature. Additionally, we also introduce a semi-private interaction framework where, while the principal publicly recommend joint actions for the agents, additional private information is used to further reshape the agents' beliefs. This form of persuasion occurs naturally when agents can see or know about each other's actions, while the principal can communicate separately with each agent through private channels.
\end{itemize}

To obtain the above results the main challenges lie in finding ways to reduce the representations of several key element that would otherwise be exponential if not handled carefully. 
In the public and semi-private cases, we introduce {\em representative action vectors} to reduce the exponentially large space of the agents' joint actions to a polynomial space, without loss of optimality.
This technique fails in the private case, as we show through a counterexample. Hence, we further introduce the concept of {\em lottery policy} by exploiting a symmetry in the agents' roles. Such lottery policies allow representative action vectors to be applicable for private persuasion, too.
Besides the representations of joint actions, concise representations of information encoded in the principal's signals are also necessary.
In the semi-private and private cases, different agents may form different posterior beliefs about the state of the world (even if they are of the same type and are recommended to perform the same action). Since our characterizations indicate that the principal's signals need to encode sufficient information that explains, to every agent, why it will not be beneficial to deviate, it appears that jointly the space of signals for us to consider grows exponentially with the number of agents. We nevertheless show a way to represent the explanations that reduce the size of the signal space to a polynomial, by defining the explanations for each deviation instead of each agent.

\subsection{Motivating Examples}

Beyond the theoretical interests, the problem of persuading multiple agents in the presence of externalities is motivated by many real-world scenarios. Consider for example a navigation app (the principal) that would like to minimize its users' overall travel time. Each user is associated with an origin and a destination (which indicates their type) and needs to choose a route (their possible actions) between these two points. The traffic load in each route depends on the state of the world. There is a prior distribution over possible states, but on any given day, the app possesses information of the actual traffic load, while individual users only have a prior knowledge about the state.

As another example, consider authorities, such as the Food and Drug Administration (FDA) in the US, that are responsible for evaluating the efficacy of new drugs. The interaction between the FDA and a company that wants their new drug to be approved can be captured by the persuasion model \cite{wang2013bayesian}. 
In this case, the company (the principal) persuades the FDA committee members (the agents) to accept the drug. The drug could be effective or ineffective, and with side effects or without (i.e., four possible states). The company needs to decide on the clinical trial to carry out, and its results will inform the committee members about the effectiveness of the drug. A committee member can vote to approve the drug, disapprove it, or ask for additional trials, which are their actions in the model.
Their utilities can depend on the actual efficacy of the drug (the state) and the impact of a right/wrong decision on their own reputation and the FDA's reputations. The weights the members put on their own and the FDA's overall reputations may differ depending on their seniority (see more details in \Cref{sc:example}).
The decision is made by majority rule, so there are externalities in the members' utilities.





\subsection{Related Work}

Bayesian persuasion against a single agent is tractable \cite{KamenicaSurvey,dughmi2017algorithmic}, but the more general problem with multiple agents becomes harder. 
Hence, most work on multi-agent Bayesian persuasion has been done under special game structures.
The constraints of each domain were used in order to find tractable solutions. For example, in the context of singleton congestion games, with both private and public signals, \cite{zhou2022algorithmic} proved that efficient computation of optimal policy for the principal is achievable when the number of resources is held constant.
\cite{xu2020tractability} explored public persuasion without externalities, featuring binary actions but with a larger number of possible states. The study established that the optimal public signaling policy can be efficiently computed in polynomial time for arbitrary principal utility functions, given a constant number of states.
\cite{castiglioni2023public} showed that it is possible to compute bi-criteria approximations of optimal public signaling schemes in arbitrary persuasion problems but, again, without externalities in quasi-poly-time.

In terms of persuasion with externalities, it has been shown that even in Bayesian zero-sum games no polynomial time approximation scheme (PTAS) is available for finding optimal public persuasion policy unless P = NP \cite{rubinstein2017honest}.
To deal with the challenges, we introduce the concept of types. 
The concept of type has also appeared in many previous works. In mechanism design settings, a type means a possible type of utility function an agent may have. The mechanism designer does not know the agent's exact type but maintains a probabilistic belief about what it might be. 
For example, \cite{castiglioni2022bayesian} 
considered a persuasion scenario where the principal elicits the agent's type first.
The agent reports their type to the principal's mechanism, and the principal commits to a signaling policy according to the agent's report. No externalities are assumed in their model. Indeed, in our model, we assume that the types of the agents are known to the principal in advance, similarly to the standard persuasion model. The concept of type is more of a structure to group agents into small categories to allow us to derive efficient algorithms. 

Our work considers three different types of signaling channels, including semi-private persuasion, which lies in between public and private persuasion. 
Similar perspectives have been adopted in the work of \cite{babichenko2021multi}, who introduced a model that smoothly transitions between public and private persuasion. 
More specifically, they introduced a multi-channel communication structure, where each agent observes a subset of the principal's communication channels. The authors provided a complete characterization of the conditions under which one communication method outperforms another. They provided characterizations applicable to scenarios both with and without externalities among the agents. 
However, their computational studies focus on the case with: a constant number of worlds, no externalities and an additive utility function of the principal. For future work, extending our model to a multi-channel structure with externalities seems very promising. 
Our semi-private model is the first step in this direction. 


\section{The Problem}

A persuasion problem is given by a tuple
$\langle \Worlds, \priors, N, \Actions, \Types, (\utility_T)_{T \in \Types \cup \{0\}}, \Maxdeviators \rangle$. 
Specifically: 
$\Worlds$ is the set of the possible worlds. 
$\priors \in \Delta(\Worlds)$ is the prior distribution of the worlds. 
$N = \{1,\dots,n\}$ contains the indices of the $n$ agents involved. $\Actions$ is the set of actions available to the agents (w.l.o.g. it is the same for all agents).
$\Types \subseteq 2^N$ is a collection of disjoint subsets---referred to as {\em types}---of the agents, which defines a partition of $N$.
Finally, $d$ is the maximum number of agents that can deviate jointly from the principal's recommendation of actions. When $d=1$, only unilateral deviations are possible, and when $d > 1$, groups deviations are allowed.

We denote by $\Actionvectors=\Actions^{n}$ the set of all possible joint actions of the agents. 
Given a joint action $\actionvector \in \Actionvectors$, we call $\rho_\actionvector \in \mathbb{N}_{\ge 0}^{|\Types| \times  |\Actions|}$ the {\em action profile} corresponding to $\actionvector$, which anonymizes the agents' identities in $\actionvector$ and encodes only the number of agents of certain types who perform certain actions: $\rho_\actionvector(T, a) = |\{i \in T: a_i = a \}|$ for all $T \in \Types$ and $a \in \Actions$.
The utility $\utility_T( a, \rho \given \omega)$ of a type-$T$ agent depends on the action $a$ performed by the agent, the action profile $\rho$ encoding the joint action of the agents, as well as the current world $\omega$. 
Alternatively, given a joint action $\actionvector$, we can write the utility of an agent $i \in T$ as:
\[
\utility_i( \actionvector \given \omega) = u_T(a_i, \rho_{\actionvector} \given \omega).
\]
Similarly, we let $u_0(\rho \given \omega)$ be the principal's utility for an action profile $\rho$ in world $\omega$ (the principal does not perform actions herself).


\subsection{The Principal's Policy}

As in standard persuasion models, the principal has an information advantage over the agents: she is the only player who can observe the world. The principal reveals this private information to influence the agents' decision-making (choices of actions). Additionally, in our model the principal also has the power to correlate the agents' actions and serves as a correlation device. 
To perform both signaling and correlation, we consider policies of the form $\sigma : \Omega \to \Delta(\Actionvectors \times \Signalvectors)$, which map the world observed by the principal to a meta-signal $s = (\actionvector, \signalvector) \in \Actionvectors \times \Signalvectors$. The joint action $\actionvector$ functions as a correlation signal and encodes the action $a_i$ each agent is recommended to perform, while $\signalvector = (g_i)_{i\in N}$ provides additional information about the world to influence the agents' beliefs, with $g_i$ sent to each agent $i$.  

We consider both public and private persuasion.
In public persuasion, the communication channel is public and the entirety of the meta-signal $s$ is seen by every agent. 
In private persuasion, each agent $i$ only observes the part of signal $(a_i, g_i)$ that is sent to them. 
In general, we use $s_i$ to denote the part of $s$ observed by agent $i$.
Hence, for every $i \in N$, we have
\begin{itemize}
\item $s_i = (\actionvector, \signalvector)$ in public persuasion, and
\item $s_i = (a_i, g_i)$ in private persuasion.
\end{itemize}
The policies used in these two cases are referred to as public and private policies, respectively.

Besides these two most commonly studied persuasion modes, we also consider an in-between mode, which we call {\em semi-private} persuasion. 
In semi-private persuasion, the joint action $\actionvector$ in the principal's signal is sent publicly while $\signalvector$ remains private. 
Essentially, private persuasion becomes semi-private when the agents can see the actions to be performed by the other agents.
Hence,
\begin{itemize}
\item $s_i = (\actionvector, g_i)$ in semi-private persuasion.
\end{itemize}
For ease of description, 
we write $s = (s_i)_{i \in N} = (\actionvector, \signalvector)$.

\subsection{Agents' Belief and Stability of the Policy}

Upon receiving a meta-signal $s_i$, the agent performs a belief update according to Bayes' rule and derives the following posterior probability for each world $\omega \in \Omega$, as well as each possible joint action ${\actionvector} \in \Actionvectors$:
\[
\Pro({\actionvector}, \omega \given s_i) = 
\frac{\mu(\omega) \cdot \sigma(s_i, {\actionvector} \given \omega)}{\sum_{\omega' \in \Omega} \mu(\omega') \cdot \sigma(s_i \given \omega')},
\]
where $\Pro$ denotes the probability measure induced by the policy, and we use the notation:
\[
\sigma(x \given \omega) = \sum_{s' \in S: s' \vdash x}  \sigma(s' \given \omega),
\]
where $s' \vdash x$ means $s'$ satisfies $x$, e.g., $s' = (\actionvector', \signalvector') \vdash (s_i, \signalvector)$ means $s'_i = s_i$ and $\signalvector' = \signalvector$, and $s' \vdash s_i$ means $s'_i = s_i$.
In the private case, the agents also maintain a probabilistic belief about the other agents' actions.

The stability of a policy is concerned with whether some subset of the agents can simultaneously benefit from some joint deviation. We only consider {\em non-transferable} utility in this paper, and we do not consider further communication between the agents (which may further change their posterior beliefs). Under these conditions, the stability of a policy is defined as follows.


\begin{definition}[Stable policy]
Given a policy $\sigma: \Omega \to S$, a signal $s = (\actionvector, \signalvector) \in S$ is {\em unstable} if there exists a set $N' \subseteq N$ containing at most $d$ agents and an action vector $\actionvector' = (a'_i)_{i \in N'}$, such that for all $i \in N'$: 
\begin{align}
\label{eq:stability}
\!\!\!
\sum_{\tilde{\actionvector} \in \Actionvectors, \omega \in \Omega} \!\!\!
\Pro(\tilde{\actionvector}, \omega \given s_i ) \cdot  
\Big( 
u_i( \tilde{\actionvector} \oplus \actionvector' \given \omega) - u_i( \tilde{\actionvector} \given \omega)
\Big)
> 0,
\end{align}
where $\tilde{\actionvector} \oplus \actionvector'$ denotes the joint action resulting from replacing $\tilde{a}_i$ with $a'_i$ for all $i \in N'$.
The signal is {\em stable} otherwise.
The policy $\sigma$ is stable if all signals in $\supp(\sigma)$ are stable.\footnote{$\supp(\cdot)$ denotes the support set, i.e., for a policy $\sigma: \Omega \to S$, $\supp(\sigma) = \{ s \in S : \sigma(s) > 0\}$.} 
\end{definition}

In the public and semi-private cases, since agents know the joint action deterministically, the instability constraint in the definition simplifies to 
\begin{align}
\label{eq:stability-pub-sem}
\sum_{\omega \in \Omega} \Pro(\omega \given s_i ) \cdot 
\Big( 
u_i( \actionvector \oplus \actionvector' \given \omega) - 
u_i( \actionvector \given \omega) 
\Big) 
> 0
\end{align}
We will sometimes abuse the notation slightly and write 
\[
u_i(\actionvector \given p) = \E_{\omega \sim p} u_i(\actionvector \given \omega)
\]
for a posterior $p \in \Delta(\Omega)$.
So, \Cref{eq:stability-pub-sem} further simplifies to $u_i( \actionvector \oplus \actionvector' \given p_i) > u_i( \actionvector \given p_i)$ by letting $p_i = \Pro(\omega \given s_i )$. 


Hence, a stable policy incentivizes the agents to perform the recommended actions. It yields utility
\[
u_0(\sigma) = \E_{\omega \sim \mu} \E_{(\actionvector, \signalvector) \sim \sigma(\cdot \given \omega)} u_0(\actionvector \given \omega)
\]
for the principal.
Our goal is to find, among all stable policies, one that maximizes the principal's utility (or decide correctly that no stable policy exists).

\subsection{The FDA Example}
\label{sc:example}

\newcommand{\rej}{\text{\textnormal{rej}}}
\newcommand{\acc}{\text{\textnormal{acc}}}

Before we present our results, we describe a concrete example to illustrate the model defined above.
Consider a simplified version of the FDA example we mentioned in the introduction. A drug can be either {\em safe} or {\em unsafe}, and the FDA decides whether to {\em accept} or {\em reject} a new drug.
The utility of the company (the principal) is $0$ if their drug is rejected, $2$ if their drug is accepted and safe, $1$ if it is accepted but unsafe.
The FDA uses the majority rule to make their decision, where each of their senior members has two votes and each junior member has one.

The states of the drugs, safe or unsafe, correspond to the two possible worlds $\omega_+$ and $\omega_-$, with priors $\priors(\omega_+) = 0.25$ and $\priors(\omega_-) = 0.75$. 
Suppose that there are $2$ senior members ($T_1 = \{1, 2\}$) and $3$ junior members ($T_2 = \{3,4,5\}$). 
The possible actions of the agents include voting for rejection ($a^\rej$) and for acceptance ($a^\acc$).
For an action profile $\rho$, the outcome of the vote is denoted by $\acc(\rho)$.

The senior members place high importance on the reputation of the FDA.
Each of them gets a high penalty ($-3$) if s/he votes in favor of an unsafe drug and the drug is accepted according to the majority vote.
A smaller penalty ($-2$) is incurred if the senior member votes against an unsafe drug but the drug is accepted. 
Similarly, utilities for the other possible cases are listed in the table below.
The junior members, on the other hand, are more concerned with their own reputation than the FDA's reputation. They get a positive utility when warning against an unsafe drug, even though the drug was eventually accepted by the majority.



\begin{table}[h]
\centering
\small
\begin{tabular}{@{\extracolsep{4pt}}rrr}
\toprule
& $\omega_-$  & $\omega_+$ \\ 
\midrule
$a^\rej,\, \acc(\rho)$ & $-2$ & $-1$ \\
$a^\acc,\, \acc(\rho)$ & $-3$ & $1$ \\
$\cdot,\, \lnot \acc(\rho)$ & $0$ & $0$ \\
\bottomrule\\[-1mm]
\multicolumn{3}{c}{$u_{T_1}(a, \rho|\omega)$}
\end{tabular}
\quad
\begin{tabular}{@{\extracolsep{4pt}}rrr}
\toprule
& $\omega_-$  & $\omega_+$ \\ 
\midrule
$a^\rej,\, \acc(\rho)$ & $2$ & $-2$ \\
$a^\acc,\, \acc(\rho)$ & $-3$ & $1$ \\
$\cdot,\, \lnot \acc(\rho)$ & $0$ & $0$ \\
\bottomrule\\[-1mm]
\multicolumn{3}{c}{$u_{T_2}(a, \rho|\omega)$}
\end{tabular}
\end{table}




Consider the action vector 
$\actionvector = (a^\rej, a^\acc, a^\acc, a^\rej, a^\acc)$.
The corresponding action profile $\rho_\actionvector$ is the function such that $\rho_\actionvector(T_1, a^\rej) = \rho_\actionvector(T_1, a^\acc) = \rho_\actionvector(T_2, a^\rej) = 1$ and $\rho_\actionvector(T_2, a^\acc) = 2$.
Since the majority votes for acceptance in this profile, the utility of the principal is $u_0(\rho_\actionvector|\omega_+)=2$ if the drug is safe; otherwise, $u_0(\rho_\actionvector|\omega_-)=1$.

\section{Public Persuasion}
\label{sc:public}

To explain our algorithm for public persuasion, we first analyze the structure of optimal policies and prove a result similar to the {\em revelation principle} \cite{kamenica2011bayesian}. 
In the case {\em without} externalities, the revelation principle states that it is without loss of optimality to consider incentive compatible (IC) and direct policies. Namely, the policy directly recommends actions to the agents and incentivizes them to perform these recommended actions. The additional signal $\signalvector$ in the meta-signal is unnecessary in this case. This elegant result simplifies the design of policies: in particular, it sets an upper bound to the size of the meta-signal space.
However, it does not hold any more when there are externalities and multiple agents can deviate together (i.e., $d>1$).
We provide an example to illustrate this.

\subsection{An Example}
\label{exm:additional-information}

Suppose that there are two worlds $\omega_1$ and $\omega_2$, with equal prior $\priors(\omega_1) = \priors(\omega_2) = 0.5$.
There are only two agents and they are of the same type, i.e., $\Types = \{T\}$ and $T = \{1, 2\}$.
There are three actions available to them: $A = \{a_1, a_2, a_3\}$. We have $d = 2$, so the agents can deviate jointly.

The agents' utilities are $-1$ for all the action profiles except two special profiles $\rho_1$ and $\rho_2$, where $\rho_1(a_1) = 2$ and $\rho_2(a_2) = \rho_2(a_3) = 1$ (we omit types since there is only one type).
The agents' utilities for these two profiles are given in the tables below. Irrespective of the world, the principal's utility is $1$ for $\rho_1$ and $0$ for all other profiles.
\begin{table}[h]
\centering
\small
\begin{tabular}{@{\extracolsep{4pt}}rrrr}
\toprule
 & $\omega_1$ & $\omega_2$ \\ 
\midrule
$a_1$ & $1$ & $1$ \\
\bottomrule\\[-2mm]
\multicolumn{3}{c}{$u_T(a, \rho_1|\omega)$}
\end{tabular}
\qquad\qquad
\begin{tabular}{@{\extracolsep{4pt}}rrrr}
\toprule
 & $\omega_1$ & $\omega_2$ \\ 
\midrule
$a_2$ & $0$ & $10$ \\
$a_3$ & $10$ & $0$ \\
\bottomrule\\[-2mm]
\multicolumn{3}{c}{$u_T(a, \rho_2|\omega)$}
\end{tabular}
\end{table}

Hence, in this example, the maximum possible expected utility of the principal is $1$, which can be obtained only if the agents are incentivized to follow $\rho_1$. It can be seen that this is achieved by the policy which always reveals the world to the agents while recommending $\actionvector^1 = (a_1, a_1)$ (which is the only action vector corresponding to $\rho_1$).
This policy is also stable. Indeed, given the utility definition, it would only be beneficial for the agents to deviate to $\rho_2$. However, in $\rho_2$, one of them will have to receive a lower utility of $0$ and hence block this joint deviation.

However, recommending $\actionvector^1$ without using any additional information is not stable. Indeed, this makes $\actionvector^1$ the only signal used in the policy, irrespective of the persuasion mode (public, semi-private, or private).
The signal is uninformative as a result, and the agents' beliefs remain the same as the prior after receiving the recommendation.
In this case they will always prefer to deviate to $\rho_2$ and obtain an expected utility of $5$. 
Therefore, for all the persuasion modes considered, signaling only the recommended actions to the agents may only produce suboptimal policies.


It turns out that while the revelation principle breaks in its original form, a variant of it can be established to bound the size of the meta-signal space. We present a new revelation-principle-style characterization next.

\subsection{Revelation Principle for Public Persuasion}

We identify meta-signals that can be merged without changing the induced outcome to reduce the meta-signal space as much as possible. To this end, we introduce the {\em signatures} of meta-signals.
We begin by defining {\em representatives} and {\em blocking profiles}.

We fix an arbitrary {\em representative set} $\Representativeactionvectors \subseteq \Actionvectors$, in which every profile $\rho \in P$ finds exactly one joint action $\bar{\actionvector} \in \Representativeactionvectors$ such that $\rho_{\bar{\actionvector}} = \rho$. Namely, $\bar{\actionvector}$ is representative of the joint actions whose action profile is $\rho$.
Since the representative set defines a bijection between the action profiles and the representative actions, we will use a representative action $\bar{\actionvector}$ and its corresponding action profile $\rho_{\bar{\actionvector}}$ interchangeably.

\begin{definition}[Representative]
A joint action $\bar{\actionvector}$ is representative of another joint action $\actionvector$ if $\bar{\actionvector} \in \Representativeactionvectors$ and $\rho_{\bar{\actionvector}} = \rho_\actionvector$.    
\end{definition}

In an action profile, we refer to the tuple $(T,a) \in \Types \times A$ as the {\em subtype} of the agents who are of type $T$ and perform action $a$.
Let $\Deviationprofiles_*$ denote the set of all possible deviations:
\[
\Deviationprofiles_* {=} 
\left\{ 
\delta : \Types {\times} A^2 \to \mathbb{Z}_{\ge 0}
\middle|\! 
\begin{array}{l}
1 \le \sum_{T, a, a'} \delta(T, a, a') \le \Maxdeviators,\\[2mm] 
\delta(T, a, a) = 0 \ \forall a, T
\end{array}
\!\! \right\}.
\]
Each deviation $\delta$ specifies the number $\delta(T, a, a')$ of agents of subtype $(T, a)$ who deviate to action $a'$. The total number of deviating agents is bounded by $d$ by assumption, while there is no need to deviate from an action to itself.
Next, let $\Deviationprofiles_\rho$ contain deviations that are feasible from $\rho$: 
\begin{align}
\label{eq:D-rho}
&\Deviationprofiles_\rho =
\left\{ 
\delta \in \Deviationprofiles_*
\middle|  
\sum\limits_{a' \in \Actions} \delta(T, a, a')  \le \rho(T, a) \ \forall a, T
\right\},
\end{align}
i.e., the number of type-$T$ agents who deviate from action $a$ must be consistent with the number of type-$T$ agents who originally perform this action in $\rho$.


The {\em blocking profile} of a public signal is then defined as follows.
It describes how each possible deviation is blocked by agents of certain subtypes.

\begin{definition}[Blocking profile]
The {\em blocking profile} of a meta-signal $s = (\actionvector, \signalvector)$ is the set $\beta =$ 
\begin{align*}
&\left\{
\begin{array}{l}
\!\!
(\delta, T,a,a') \in \\
\quad D_\rho {\times} \Types {\times} A^2 
\end{array}
\!\middle|\!
\begin{array}{l}
\delta(T,a, a') > 0, \\[2mm]
u_T(a, \rho_\actionvector \given p_s) \ge u_T(a', \rho_\actionvector \oplus \delta \given p_s)
\end{array}
\!\!\right\},
\end{align*}
where 
$p_s = \Pro(\cdot \given s)$ denotes the posterior induced by $s$, and $\rho_\actionvector \oplus \delta$ denotes the action profile resulting from applying deviation $\delta$ to $\rho_\actionvector$.
\end{definition}

Namely, $(\delta, T, a, a') \in \beta$ means that the agents of subtype $(T,a)$ are unwilling to deviate to playing $a$, hence blocking $\delta$.
It is straightforward that if a meta-signal is stable, then its blocking profile covers every $\delta \in D_\rho$, i.e., it contains at least one element involving $\delta$.
We denote by $\pubblockingprofiles_\rho$ the set of subsets of $D_\rho \times \Types \times A^2$ that covers every $\delta \in D_\rho$ in the public case. In other words, $\pubblockingprofiles_{\rho_\actionvector}$ consists of all possible blocking profiles of stable meta-signals $s = (\actionvector, \signalvector)$.
For simplicity, we will also write $\pubblockingprofiles_{\rho_\actionvector}$ as $\pubblockingprofiles_{\actionvector}$.

Combining a representative and a blocking profile gives the signature of a meta-signal.

\begin{definition}[Signature]
\label{def:typed-sig}
The {\em public signature} of a meta-signal $s = (\actionvector, \signalvector)$, denoted $\sigpub(s)$, is the tuple $(\bar{\actionvector}, \beta)$ where $\bar{\actionvector}$ is the representative of $\actionvector$ and $\beta$ is the blocking profile of $s$.
\end{definition}


Intuitively, meta-signals with the same signature can be merged without affecting the outcome induced by the policy. The support of the resulting policy is then bounded by the number of distinct signatures there are.
This leads to the following revelation-principle-style characterization for public persuasion.

\begin{restatable}{theorem}{thmRpPublic}
\label{thm:rp-public}
For any stable public policy $\sigma : \Omega \to \Delta(\Actionvectors \times G)$, there exists a stable public policy $\bar{\sigma} : \Omega \to \Delta(C)$ that yields as much utility for the principal as $\sigma$ does, where 
$C = \{ (\bar{\actionvector}, \beta) : \bar{\actionvector} \in \Representativeactionvectors, \beta \in \pubblockingprofiles_{{\bar{\actionvector}}} \}$. Moreover, the signature of each meta-signal $c \in C$ is exactly $c$.\footnote{Omitted proofs can be found in the appendix.}
\end{restatable}

The above result indicates that it suffices to encode in each meta-signal the recommended actions, as well as a blocking profile to explain to the agents why it would not be beneficial for them to deviate in any possible way.
The signature of every meta-signal used in the policy is exactly the meta-signal itself.
Based on the theorem, we design an algorithm to compute optimal public policies next.

\subsection{Computing an Optimal Public Policy}

We use the following LP to compute an optimal public policy, with variables
$\sigma(s_{\bar{\actionvector}, \beta} \given \omega)$ for $\bar{\actionvector} \in \Representativeactionvectors, \beta \in \pubblockingprofiles_{{\bar{\actionvector}}}, \omega \in \Omega$.
\begin{align}
\label{eq:lp-pub-obj}
    \max \quad 
    \sum_{\omega \in \Worlds} \priors(\omega) \sum_{\bar{\actionvector} \in \Representativeactionvectors}  \sum_{\beta \in \pubblockingprofiles_{{\bar{\actionvector}}}} \sigma(s_{\bar{\actionvector}, \beta} \given \omega) \cdot \utility_0(\bar{\actionvector} \given \omega),
\end{align}
subject to the following constraint
for every $\bar{\actionvector} \in \Representativeactionvectors$, 
$\beta \in \pubblockingprofiles_{{\bar{\actionvector}}}$, and every tuple $(\delta, T, a, a') \in \beta$:
\begin{align}
    \sum_{\omega \in \Worlds} \priors(\omega) \cdot \sigma(s_{\bar{\actionvector}, \beta} \given \omega) \cdot \utility_{T}(a', \rho_{\bar{\actionvector}} \oplus \delta \given \omega) \le \nonumber \\ 
    \sum_{\omega \in \Worlds} \priors(\omega) \cdot \sigma(s_{\bar{\actionvector}, \beta} \given \omega) \cdot \utility_{T}(a, \rho_{\bar{\actionvector}} \given \omega).
    \label{eq:lp-pub-cons-1}
\end{align}
Additionally, for all $\omega \in \Worlds$:
\begin{align}
\label{eq:sigma-valid-1-public}
&\sum_{\bar{\actionvector} \in \Representativeactionvectors} \sum_{\beta \in \pubblockingprofiles_{\bar{\actionvector}}} \sigma(s_{\bar{\actionvector}, \beta} \given \omega) = 1 \\
\label{eq:sigma-valid-2}
& \sigma(s_{\bar{\actionvector}, \beta} \given \omega) \ge 0 \quad \text{for all } \bar{\actionvector} \in \Representativeactionvectors, \beta \in \pubblockingprofiles_{\bar{\actionvector}}
\end{align}
so that $\sigma(\cdot \given \omega)$ is a valid distribution.

In other words, the variables encode a policy characterized by \Cref{thm:rp-public}.
Under the assumption that this policy is stable, the objective function \Cref{eq:lp-pub-obj} captures exactly the expected utility of the principal. 
\Cref{eq:lp-pub-cons-1} further ensures that the signature of each meta-signal is exactly itself, which follows by (the negation of) \Cref{eq:stability-pub-sem}.
Consequently, the policy is indeed stable since by definition $\pubblockingprofiles_{\rho}$ only contains blocking profiles under which $\rho$ is stable.





\paragraph{Time Complexity}

Given a constant $d$, and a constant number of types and actions in $\Types$ and $A$, the size of the above LP grows only polynomially in the problem size. 
For general $d$, however, the formulation can grow exponentially. Indeed, this is inevitable: when $d$ is part of the input, computing an optimal public policy is NP-hard. The reduction also applies to the semi-private and private cases.

\begin{restatable}{theorem}{thmOptSubPolyTime}
    \label{thm:opt-pub-poly-time}
     For constant $d$, $|\Types|$, and $|A|$,
    an optimal public policy can be computed in polynomial time.
\end{restatable}

\begin{restatable}{theorem}{thmHardness}
\label{thm:public-hardness}
    When $d$ is an input to the problem, computing an optimal policy is NP-hard for public, semi-private, and private persuasion.
\end{restatable}


\section{Semi-private Persuasion}
\label{sc:semi-private}


We now consider the semi-private case. Similarly to the public case, a blocking profile should contain information about how each possible deviation is blocked. However, in the semi-private case we need to define the profiles separately for different agents because of divergent beliefs. 
This may result in an exponential growth of blocking profiles when the number of agents increases (i.e., each profile corresponds to a possible combination of $n$ sets of deviations).

\subsection{Representing Semi-Private Blocking Profiles}
To address the above issue, we explore the minimum possible representation of a blocking profile. 
Key to our approach is
\Cref{lmm:hall-variant}, a generalization of Hall's theorem \cite{hall1934representation}.
With this result, we can prove \Cref{lmm:blocking-conditions} to find a function $\gamma$ that encodes sufficient information for us to derive a concise representation of blocking profiles.



\begin{restatable}{lemma}{lmmhallvariant}
\label{lmm:hall-variant}
Suppose that
$B_1, \dots, B_m \subseteq \{1, \dots, \ell\}$, and $r_1, \dots, r_m \in \mathbb{R}$ are non-negative.
There exist $m$ disjoint sets $\widetilde{B}_1, \dots, \widetilde{B}_m$, $\widetilde{B}_i \subseteq B_i$ and $|\widetilde{B}_i| = r_i$ for all $i=1,\dots,m$, if and only if it holds for every set $M \subseteq \{1,\dots,m\}$ that $|\bigcup_{i \in M} B_i| \ge \sum_{i \in M} r_i$.
\end{restatable}

\begin{restatable}{lemma}{lmmblockingconditions}
\label{lmm:blocking-conditions}
If a meta-signal $s = (\actionvector, \signalvector)$ is stable, then there exists a function
$\gamma: \Deviationprofiles_* \to 2^\Actions \times 2^\Agents$ such that for every $\delta$, it holds for $(A', N') = \gamma(\delta)$ that:
\begin{itemize}
\item[1)] 
agents in $N'$ are all of the same subtype, say $(T,a)$;

\item[2)] 
$\delta(T, a, a') > 0$, for all $a' \in A'$;

\item[3)] 
$\rho_\actionvector(T, a) - |N'| < \sum_{a' \in A'} \delta(T, a, a')$; 
and

\item[4)] for every $a' \in A', i \in N'$,
\begin{align}
\label{eq:sm-blocking-profile-uT}
u_T(a, \rho_\actionvector \given p_i) \ge u_T(a', \rho_\actionvector \oplus \delta \given p_i),
\end{align}
where $p_i = \Pro(\cdot \given s_i)$ denotes the posterior induced by $s_i$.
\end{itemize}    
\end{restatable}

Intuitively, for a deviation to be successful, a matching that fulfills this deviation needs to be established between agents of certain subtypes and the actions they are willing to deviate to. 
A stable meta-signal prevents such matchings.
Hence, for every possible deviation $\delta \in D_{\rho_\actionvector}$, the set $N'$ stated in the lemma contains a set of agents of subtype $(T, a)$ who are unwilling to deviate to any of the actions in $A'$ according to \Cref{eq:sm-blocking-profile-uT}.
According to the third condition in the lemma, $N'$ is large enough, so that the remaining agents of this subtype are insufficient for fulfilling $\delta$.
Based on $\gamma$, we define the semi-private blocking profiles as follows.

\begin{definition}[Semi-private blocking profile]
\label{def:sm-blocking-profile}



Let $s = (\actionvector, \signalvector)$ be a stable meta-signal and $\gamma$ be the function in \Cref{lmm:blocking-conditions}.\footnote{If there are multiple such functions, we choose one according to an arbitrary tie-breaking rule. The same applies to \Cref{def:pr-blocking-profile}.}
The {\em semi-private blocking profile} of each $s_i$ is
\[
\beta_i = 
\left\{ 
(\delta, A') \in D_{\rho_\actionvector} \times 2^A \,\middle|\, 
i \in N' \text{ for } (A', N') = \gamma(\delta)
\right\}.
\]
The semi-private blocking profile of $s$ is $\beta = (\beta_i)_{i \in N}$.
\end{definition}


Namely, $\beta_i$ describes, for every blocked deviation $\delta$,
a set $A'$ of actions disliked by agent $i$ (as well as other agents in some $N'$) that results in $\delta$ being blocked.
This is the minimum information that ``witnesses'' the stability of $s$ while ensuring that merging meta-signals with the same semi-private signature (\Cref{def:semi-typed-sig}) does not change the outcome of the policy. 
While the sets $N'$ are not listed in any $\beta_i$, it is implicitly defined through the joint profile $\beta$.
This is key to avoiding the exponential growth of the profile space.

\begin{definition}[Semi-private signature]
\label{def:semi-typed-sig}
The {\em semi-private signature} of a meta-signal $s = (\actionvector, \signalvector)$, denoted $\sigsm(s)$, is the tuple $(\bar{\actionvector}, \beta)$ where $\bar{\actionvector}$ is a representative of $\actionvector$, and $\beta = (\beta_i)_{i \in N}$ is the semi-private blocking profile of $s$.
For each $s_i = (\actionvector, g_i)$, the semi-private signature of $s_i$ is defined as $\sigsm(s_i) = \sigsm_i(s) = (\bar{\actionvector}, \beta_i)$.
\end{definition}

We can now prove a new revelation principle for the semi-private case.
For every $\actionvector \in \Actionvectors$, we let $\smblockingprofiles_\actionvector$ be the set of blocking profiles defined by all possible meta-signals according to \Cref{def:sm-blocking-profile}, when the joint action in the meta-signal is fixed to $\actionvector$ (which also appear in the last two conditions about $\gamma$ in \Cref{lmm:blocking-conditions}).

\begin{restatable}{theorem}{thmRpSemiPrivate}
\label{thm:rp-semi-private}
For any stable semi-private policy $\sigma : \Omega \to \Delta(\Actionvectors \times G)$, there exists a stable semi-private policy $\bar{\sigma} : \Omega \to \Delta(C)$ that yields as much utility for the principal as $\sigma$ does, where 
$C = \{ (\bar{\actionvector}, \beta) : \bar{\actionvector} \in \Representativeactionvectors, \beta \in \smblockingprofiles_{\bar{\actionvector}} \}$.
Moreover, the semi-private signature of each meta-signal $c \in C$ is $c$.
\end{restatable}

It can be shown that when $d$, $|\Types|$, and $|A|$ are constants, $\smblockingprofiles_{\bar{\actionvector}}$ grows only polynomially with the size of the problem. 
Similarly to the public case, we obtain a polynomial-size LP formulation (see \Cref{sec:lp-semi}).
This leads to a polynomial-time algorithm for semi-private persuasion.

\begin{restatable}{theorem}{thmOptSmPolyTime}
    \label{thm:opt-sm-poly-time}
    For constant $d$, $|\Types|$, and $|A|$, an optimal semi-private policy can be computed in polynomial time.
\end{restatable}





\section{Private Persuasion}
\label{sc:private}

In private persuasion, even the action recommendations are sent privately. This means that the agents maintain probabilistic beliefs about both the world and the actions of the other agents.
We follow the same routine as the previous sections, deriving a revelation principle for the private case and then an LP formulation based on the result.
A crucial difference in the private case is that it is no longer w.l.o.g. to restrict the policy to representative action vectors (\Cref{prp:prv-representative-not-opt}).\footnote{This also implies that private persuasion can achieve strictly higher utilities than semi-private and public persuasion, since in the latter two it is w.l.o.g. to consider representative action vectors.}
To solve the problem requires a novel concept called the {\em lottery policy}.

\begin{restatable}{proposition}{thmPrvRepresentativeNotOpt}
    \label{prp:prv-representative-not-opt}
    An optimal private policy supported only on representative action vectors need not exist.
\end{restatable}

Before we define lottery policies, we first extend the notions of blocking profile and signature to the private case. 
We can replicate \Cref{lmm:blocking-conditions} and prove the following lemma, with slight changes in the domain of $\gamma$ and the stability constraint in the last condition. In the private case, the agents' expected utilities are functions of their beliefs over all possible joint actions. 

\begin{restatable}{lemma}{lmmblockingconditionsprivate}
\label{lmm:blocking-conditions-private}
If a meta-signal $s = (\actionvector, \signalvector)$ is stable, then there exists a function
$\gamma: \Deviationprofiles_* \to 2^\Actions \times 2^\Agents$ such that for every $\delta$, it holds for $(A', N') = \gamma(\delta)$ that:
\begin{itemize}
\item[1)] 
agents in $N'$ are all of the same subtype, say $(T,a)$;

\item[2)] 
$\delta(T, a, a') > 0$, for all $a' \in A'$;

\item[3)] 
$\rho_\actionvector(T, a) - |N'| < \sum_{a' \in A'} \delta(T, a, a')$; 
and

\item[4)] for every $a' \in A', i \in N'$,
\begin{align*}
\label{eq:pr-blocking-profile-uT}
& \sum_{\omega \in \Worlds}\
\sum_{\tilde{\actionvector} \in \Actionvectors: \tilde{a}_i = a} 
\Pro(\tilde{\actionvector}, \omega \given s_i) \cdot
u_T(a, \rho_{\tilde{\actionvector}} \given \omega)
\ge \nonumber \\
&\qquad
\sum_{\omega \in \Worlds}\
\sum_{\tilde{\actionvector} \in \Actionvectors: \tilde{a}_i = a} 
\Pro(\tilde{\actionvector}, \omega \given s_i) \cdot
u_T(a', \rho_{\tilde{\actionvector}} \oplus \delta \given \omega).
\end{align*}  
\end{itemize}
\end{restatable}

This leads to the definitions of the private blocking profile and the private signature.

\begin{definition}[Private blocking profile]
\label{def:pr-blocking-profile}



Let $s = (\actionvector, \signalvector)$ be a stable meta-signal.
The {\em private blocking profile} of each $s_i$ is 
\[
\beta_i = 
\left\{ 
(\delta, A') \in D_* \times 2^A \,\middle|\, 
i \in N' \text{ for } (A', N') = \gamma(\delta)
\right\},
\]
where $\gamma$ is the function satisfying the conditions in \Cref{lmm:blocking-conditions-private}.
\end{definition}


\begin{definition}[Private signature]
\label{def:private-typed-sig}
The {\em private signature} of a meta-signal $s = (\actionvector, \signalvector)$, denoted $\sigpv(s)$, is $(\actionvector, \beta)$ where $\beta = (\beta_i)_{i \in N}$ is the private blocking profile of $s$.
For each $s_i = (a_i, g_i)$, the private signature of $s_i$ is defined as $\sigpv(s_i) = \sigpv_i(s) = (\actionvector_i, \beta_i)$.
\end{definition}


\subsection{Lottery Policy and Revelation Principle}


We now introduce lottery policies. 
Lottery policies utilize a symmetry in the agents' roles to avoid explicit representation of the original policies.
Given a private policy $\sigma$, a lottery policy $\lotterypolicy(\sigma)$ ``lotterizes'' $\sigma$ by uniformly randomizing the signals of $\sigma$ among agents of the same type.

\begin{definition}[Lottery policy]
Let $\Permutations$ be the set of all permutations $m: N \to N$ such that $m(i) \in T$ if and only if $i \in T$.
When $\sigma$ signals each meta-signal $(\actionvector, \signalvector)$, the lottery policy $\lotterypolicy(\sigma)$ draws $m \sim \mathrm{Uniform}(\Permutations)$ and signals $(\actionvector', \signalvector')$, where $a'_i = a_{m(i)}, g'_i=g_{m(i)}$.
\end{definition}

For example, consider an instance with two types of agents: $T_1 = \{1,2\}$ and $T_2 = \{3\}$; and a private policy $\sigma$ such that:
\begin{align*}
&\sigma((a_1, a_2, a_3), (g_1, g_1, g_1) \given \omega_1) = 0.4 \\
\text{and}\quad
&\sigma((a_1, a_2, a_1), (g_1, g_2, g_2) \given \omega_1) = 0.6. 
\end{align*}
The corresponding lottery policy gives:
\begin{align*}
&\lambda(\sigma)((a_1, a_2, a_3), (g_1, g_1, g_1) \given \omega_1) = 0.2, \\
&\lambda(\sigma)((a_2, a_1, a_3), (g_1, g_1, g_1) \given \omega_1) = 0.2, \\
&\lambda(\sigma)((a_1, a_2, a_1), (g_1, g_2, g_2) \given \omega_1) = 0.3, \\
\text{and}\quad
&\lambda(\sigma)((a_2, a_1, a_1), (g_2, g_1, g_2) \given \omega_1) = 0.3. 
\end{align*}


Using lottery policies, we can derive a concise characterization of optimal private policies, as stated in \Cref{thm:rp-private}. 
We let $\pvblockingprofiles_\actionvector$ be the set of blocking profiles induced by all possible meta-signals according to \Cref{def:private-typed-sig}.

\begin{restatable}{theorem}{thmRpPrivate}
    \label{thm:rp-private}
     There exists a private policy $\sigma: \Worlds \rightarrow \Delta(C)$ where $C = \{ (\bar{\actionvector}, \beta) : \bar{\actionvector} \in \Representativeactionvectors, \beta \in \pvblockingprofiles_{\bar{\actionvector}} \}$, such that $\lotterypolicy(\sigma)$ is an optimal private policy. Moreover, the signature of each meta-signal $c \in C$ is exactly the $c$.
\end{restatable}



Similarly to the semi-private case, when $d$, $|\Types|$, and $|A|$ are constants, $\pvblockingprofiles_{\bar{\actionvector}}$ grows only polynomially with the size of the problem, so we can derive a polynomial-size LP formulation, which gives a polynomial-time algorithm for private persuasion.

\begin{restatable}{theorem}{thmOptPrPolyTime}
    \label{thm:opt-pr-poly-time}
    For constant $d$, $|\Types|$, and $|A|$, an optimal private policy can be computed in polynomial time.
\end{restatable}



\section{Acknowledgments}

This research has been partially supported by the Israel Science
Foundation under grant  2544/24 and the Israel Ministry of  Innovation, Science \& Technology grant  1001818511.

\section{Conclusion}

We studied the problem of Bayesian persuasion with externalities. We showed that when multiple agents can coordinate their deviation the classical revelation principle does not hold, and presented alternative characterizations of optimal policies for public, private, and semi-private persuasion.  The concept of agent types is introduced as a succinct representation of the problem. We presented polynomial time algorithms when only a constant number of agents can jointly deviate, and we proved that the problem is hard otherwise.


\bibliography{refs}

\clearpage







\appendix

\setcounter{secnumdepth}{2}

\section{Omitted Formulations}

\subsection{LP Formulation for Semi-Private Persuasion}

\label{sec:lp-semi}

The following LP computes an optimal semi-private policy.
The variables of the LP are $\sigma(s_{\bar{\actionvector}, \beta}|\omega)$ for $\bar{\actionvector} \in \Representativeactionvectors, \beta \in \smblockingprofiles_{\bar{\actionvector}}, \omega \in \Worlds \}$; they encode a policy with signal space $\{ (\bar{\actionvector}, \beta) : \bar{\actionvector} \in \Representativeactionvectors, \beta \in \smblockingprofiles_{\bar{\actionvector}} \}$, as described in the statement of \Cref{{thm:rp-semi-private}}.
The LP is as follows:
\begin{align}
    \max\quad \sum_{\omega \in \Worlds} \priors(\omega) \sum_{\bar{\actionvector} \in \Representativeactionvectors}  \sum_{\beta \in \smblockingprofiles_{\bar{\actionvector}}} \sigma(s_{\bar{\actionvector}, \beta}|\omega) \cdot \utility_0(\bar{\actionvector}|\omega),
\end{align}
subject to the following constraint
for every $\bar{\actionvector} \in \Representativeactionvectors$, 
$\beta \in \smblockingprofiles_{\bar{\actionvector}}$, $i \in N$, every tuple $(\delta, A') \in \beta_i$, and every $a' \in A'$:
\begin{align}
    \sum_{\omega \in \Worlds} \priors(\omega) \sum_{\beta' \in \smblockingprofiles_{\bar{\actionvector}}: \beta'_i = \beta_i} \sigma(s_{\bar{\actionvector}, \beta'} \given \omega) \cdot \utility_{T}(a', \rho_{\bar{\actionvector}} \oplus \delta \given \omega)  \nonumber \\ 
    \le 
    \sum_{\omega \in \Worlds} \priors(\omega) \sum_{\beta' \in \smblockingprofiles_{\bar{\actionvector}}: \beta'_i = \beta_i} \sigma(s_{\bar{\actionvector}, \beta'} \given \omega) \cdot \utility_{T}(a, \rho_{\bar{\actionvector}} \given \omega).
\label{eq:sm-lp-cons-1}
\end{align}
Additionally, for all $\omega \in \Worlds$, we impose constraints similar to \Cref{eq:sigma-valid-1-public,eq:sigma-valid-2}
to ensure that $\sigma(\cdot \given \omega)$ is a valid distribution.

The formulation is similar to that for the public case. In particular, \Cref{eq:sm-lp-cons-1} ensures that \Cref{eq:sm-blocking-profile-uT} holds so that the semi-private blocking profile of each meta-signal $(\bar{\actionvector}, \beta)$ is exactly $\beta$.




\subsection{LP Formulation for Private Persuasion}

The following LP computes an optimal private policy. 
The variables of the LP are $\sigma(s_{\bar{\actionvector}, \beta}|\omega)$, for $\bar{\actionvector} \in \Representativeactionvectors, \beta \in \pvblockingprofiles_{\bar{\actionvector}}, \omega \in \Worlds$, which encode a policy $\sigma$ with signals in $\{ (\bar{\actionvector}, \beta) : \bar{\actionvector} \in \Representativeactionvectors, \beta \in \pvblockingprofiles_{\bar{\actionvector}} \}$. 
\begin{align}
    \max \sum_{\omega \in \Worlds} \priors(\omega) \sum_{\bar{\actionvector} \in \Representativeactionvectors}  \sum_{\beta \in \pvblockingprofiles_{\bar{\actionvector}}} \sigma(s_{\bar{\actionvector}, \beta}|\omega) \cdot \utility_0(\bar{\actionvector}|\omega),
\end{align}
subject to the following constraint
for every $\bar{\actionvector} \in \Representativeactionvectors$, 
$\beta \in \pvblockingprofiles_{\bar{\actionvector}}$, $i \in N$ of type $T$, every tuple $(\delta, A') \in \beta_i$, and every $a' \in A'$:
\begin{align}
\label{eq:pr-lp-cs-1}
&\sum_{\omega \in \Worlds} \priors(\omega) 
\!\!\!\!
\sum_{\substack{\bar{\actionvector}' \in \Representativeactionvectors, j \in T, \beta' \in \pvblockingprofiles_{\bar{\actionvector}'}: \\ \bar{a}'_j = \bar{a}_j, \beta'_j = \beta_j}} \sigma(s_{\bar{\actionvector}', \beta'} \given \omega) \cdot \utility_{T}(a', \rho_{\bar{\actionvector}'} \oplus \delta \given \omega) \le \nonumber \\ 
& 
\sum_{\omega \in \Worlds} \priors(\omega) 
\!\!\!\!
\sum_{\substack{\bar{\actionvector}' \in \Representativeactionvectors, j \in T, \beta' \in \pvblockingprofiles_{\bar{\actionvector}'}: \\ \bar{a}'_j = \bar{a}_j, \beta'_j = \beta_j}}\sigma(s_{\bar{\actionvector}', \beta'} \given \omega) \cdot \utility_{T}(\bar{a}'_j, \rho_{\bar{\actionvector}'} \given \omega).
\end{align}
Additionally, for all $\omega \in \Worlds$, we impose constraints similar to \Cref{eq:sigma-valid-1-public,eq:sigma-valid-2}
to ensure that $\sigma(\cdot \given \omega)$ is a valid distribution.

The equations are similar to those in the LP for semi-public persuasion. A slight difference is that \Cref{eq:pr-lp-cs-1} enforces $\lotterypolicy(\sigma)$ to be stable, instead of $\sigma$. Hence, the summations are over $j \in T$ as the lottery policy randomized the roles of agents of the same type.
The agents maintain posterior beliefs, upon seeing $s_i$, about their role selected by $\lotterypolicy(\sigma)$.
Hence, the constraint ensures that \Cref{eq:pr-blocking-profile-uT} holds for $\lotterypolicy(\sigma)$.



\section{Omitted Proofs}

\subsection{Proofs in \Cref{sc:public}}

{\renewcommand\footnote[1]{}\thmRpPublic*}

\begin{proof}
    Given $\sigma$, we construct $\bar{\sigma}$ by merging meta-signals with the same signature.
    Specifically, for every ${\actionvector} \in \Actionvectors$ and $\beta \in \pubblockingprofiles_{{\actionvector}}$, we let 
    \begin{equation}
    \label{eq:construct-sigma-prime-public}
    \bar{\sigma}( (\bar{\actionvector}, \beta) \given \omega) = \sum_{s: \sigpub(s) = (\bar{\actionvector}, \beta)} \sigma( s \given \omega).
    \end{equation}
    Namely, whenever the original policy generates a meta-signal $s$, the new policy generates the signature $(\bar{\actionvector}, \beta)$ of $s$ as a meta-signal to send.

    We verify the following to complete the proof: 1) $\bar{\sigma}$ is a valid policy, 2) $\bar{\sigma}$ is stable, and 3) $\bar{\sigma}$ yields the same utility for the principal.

    \paragraph{Validity}
    First, $\bar{\sigma}$ is valid because different signatures correspond to disjoint sets of meta-signals, which means that $\sum_{(\bar{\actionvector}, \beta) \in \bar{S}} \bar{\sigma}( (\bar{\actionvector}, \beta) \given \omega) = 1$.

    \paragraph{Stability}
    Consider an arbitrary meta-signal $c = (\bar{\actionvector}, \beta)$ in the support of $\bar{\sigma}$.
    We first argue that after we merge the meta-signals, the posterior induced by $c$ is a convex combination of the posteriors of the merged meta-signals.
    Specifically, let $\bar{\Pro}$ denote the probability measure induced by $\bar{\sigma}$; for every world $\omega$, we have
    \begin{align*}
    \bar{\Pro}(\omega \given c)
    &= 
    \frac{\mu(\omega) \cdot \bar{\sigma}(c \given \omega)}{ \bar{\Pro}(c)} \\
    &= 
    \frac{\mu(\omega) \sum_{s: \phi(s) = c} 
    \sigma(s \given \omega)}{ \bar{\Pro}(c)} \\
    &=
    \sum_{s: \phi(s) = c}
    \frac{\Pro(s)}{\bar{\Pro}(c)} \cdot 
    \frac{\mu(\omega) \cdot \sigma(s \given \omega)}{\Pro(s)} \\
    &=
    \sum_{s: \phi(s) = c}
    \frac{\Pro(s)}{\bar{\Pro}(c)} \cdot 
    \Pro(\omega \given s).
    \end{align*}
    The coefficients satisfy
    $\sum_{s: \phi(s) = c}
    \frac{\Pro(s)}{\bar{\Pro}(c)}  = 1$ given the construction in \Cref{eq:construct-sigma-prime-public}. Hence, $\bar{\Pro}( \cdot \given c)$ is a convex combination of $\mathcal{P} = \{ \Pro( \cdot \given s) : \phi(s) = c \}$.

    This implies that the signature of $c$ (as a meta-signal of $\bar{\sigma}$) is exactly $(\bar{\actionvector}, \beta)$. 
    Indeed, by definition, we have
    \begin{align*}
    &(\delta, T, a, a')
    \in \beta
    \ \Longleftrightarrow\ \\
    &\qquad \delta(T,a, a') > 0 \wedge
    u_T(a, \rho \given p) \ge u_T(a', \rho \oplus \delta \given p)
    \end{align*}
    for all $p \in \mathcal{P}$.
    The utility functions on the two sides of the inequality are linear in $p$.
    So, the inequality holds for $\bar{\Pro}(\omega \given c)$ if and only if it holds for all points in $\mathcal{P}$.
    
    By assumption $\sigma$ is stable, so every $s$ in the support of $\sigma$---particularly those with signature $(\bar{\actionvector},\beta)$---is stable.
    This implies $\beta \in \pubblockingprofiles_{\actionvector}$; hence, $c$ is stable.

    \paragraph{Utility Equality}
    Since $\bar{\sigma}$ is stable, the agents are incentivized to perform the recommended actions.
    By construction, the original policy $\sigma$ and the new policy $\bar{\sigma}$ always recommend the same actions. Hence, the utilities these two policies yield are the same.
\end{proof}

\thmOptSubPolyTime*

\begin{proof}
    Given the constants, the number of possible functions in $D_\rho$ is a constant: the domain $\Types \times A^2$ has a constant size, while the range of the function is $\{1,\dots,d\}$, which also has a constant size. 
    Hence, the space of blocking profiles, which contains subsets of $D_\rho \times \Types \times A^2$, is also constant.
    As a result, the size of the LP formulation is polynomial in the problem size.  
\end{proof}

\thmHardness*

\begin{proof}

The proof is via a reduction from the {\sc Vertex Cover} problem. Given a (undirected) graph $G = (V, E)$ and an integer $k$, {\sc Vertex Cover} asks whether there exists a size-$k$ {\em vertex cover} on $G$, i.e., a set $V' \subseteq V$ such that $i\in V'$ or $j \in V'$ holds for every edge $\{i,j\} \in E$.

We reduce a {\sc Vertex Cover} instance $\langle G = (V, E), k \rangle$ to the following persuasion instance.
Let there be $|V| + 1$ worlds, so $\Worlds = \{\omega_i: i \in V\} \cup \{\omega_0\}$.
Let the prior $\priors$ be a uniform distribution over $\Worlds$.
Let there be four actions: $\Actions = \{a_1, a_2, a_3, a_4\}$.
There is a single agent type, $\Types = \{T\}$, consisting of $n = \max(|E| + 2,\, 2|V| + 2)$ agents. The maximum number of agents that can deviate jointly is $\Maxdeviators = n - 1$.

To define the utility functions, we first define the following classes of action profiles. We define the next action profiles, for each $i \in \{1,\dots, |E|\}$, $j \in \{1, \dots, |V|+1\}$:
\begin{table}[h]
\centering
\small
\begin{tabular}{@{\extracolsep{4pt}}rrrrr}
\toprule
& $a_1$ & $a_2$ & $a_3$ & $a_4$ \\ 
\midrule
$\rho_{i}$ & $n-i-2$ & $1$ & $1$ & $i$ \\
$\rho'_{j}$ & $n-j-1$ & $1$ & $0$ & $j$ \\
$\rho^+$ & $n$ & $0$ & $0$ & $0$ \\
$\rho^-$ & $0$ & $n-2$ & $1$ & $1$ \\
\bottomrule
\end{tabular}
\end{table}

\noindent
Namely, each entry is the number of agents who perform the action at the column in the profile at the row (we omit agent type in the profiles as there is only one type).
We let the principal's utility function be: 
\begin{align*}
u_0(\rho_+ \given \omega) = 1 
\quad \text{and} \quad
u_0(\rho \given \omega) = 0
\end{align*}
for all $\omega \in \Omega$ and all $\rho \neq \rho^+$, i.e., $\rho^+$ is preferred to all other profiles, irrespective of the world.
The agents' utilities are further defined as:
\begin{align*}
u_{T, a}(\rho_+ \given \omega) = 0
\quad \text{and} \quad
u_{T, a}(\rho_- \given \omega) = 3
\end{align*}
for all $\omega \in \Omega, a \in \Actions$; moreover,
\begin{align*}
& u_{T, a}(\rho'_j \given \omega_0) = -1,\\ 
& u_{T, a}(\rho'_j \given \omega_j) = 1, \\
\text{and} \quad
& u_{T, a}(\rho'_j \given \omega_k) = 0
\end{align*}
for all $a \in \Actions$ and $j,k$ such that $j \ne k$ and $1 \le j, k \le |V|$.

Let $e_\ell$ be the $\ell$-th edge in $E$. 
For each edge $e_\ell = \{i,j\} \in E$, $i < j$, we let 
\begin{align*}
&
u_{T, a_2}(\rho_{\ell}|\omega_0) = 1,
\quad
u_{T, a_2}(\rho_{\ell}|\omega_i) = -1 \\
&
u_{T, a_3}(\rho_{\ell}|\omega_0) = 1,
\quad
u_{T, a_3}(\rho_{\ell}|\omega_j) = -1 \\
&
u_{T, a_2}(\rho_{\ell}|\omega') = 0,
\quad
u_{T, a_3}(\rho_{\ell}|\omega'') = 0
\end{align*}
for all $\omega' \neq \omega_i$ and $\omega'' \neq \omega_j$. 

Moreover, we let
\begin{align*}
u_{T, a}(\rho'_{|V| + 1}|\omega_0) = -k, 
\quad \text{and} \quad
u_{T, a}(\rho'_{|V| + 1}|\omega_i) = 1 
\end{align*}
for all $a \in A$ and $i \in \{1,\dots,|V|\}$.

Finally, we let, for $\omega \in \Worlds$ and $1 \le i \le |V| + 1$:
\[
u_{T, a_4}(\rho'_i|\omega) = 2
\]
And for $1 \le i \le |V|$:
\begin{align*}
    & u_{T, a_2}(\rho'_i|\omega_0) = -1
    & \\
    & u_{T, a_2}(\rho'_i|\omega_i) = 1
\end{align*}
And for $1 \le i, j \le |\Omega|, j \ne i$:
\begin{align*}
    & u_{T, a_2}(\rho'_i|\omega_j) = 0
    & \\
    & u_{T, a_3}(\rho'_i|\omega_j) = 0
\end{align*}
We let both types' utilities for any other undefined cases be $-k-1$.

To see that the reduction is correct, we demonstrate that the {\sc Vertex Cover} instance admits a solution of size-$k$ if and only if the principal can obtain a positive utility by using some policy in the above persuasion problem.

We begin by indicating that due to the defined principal's utility, in every optimal policy only the action vectors of $\rho_+, \rho_-$ will be signaled, since from every other profile agents will deviate into $\rho_-$.

For every policy $\sigma$ that signals $\rho_+$ with a positive utility:

Agents will deviate from $\rho_+$ into $\rho'_i$ for $1 \le i \le |V|$ if and only if there exists an agent whose belief in $\omega_i$ is stronger than its belief in $\omega_0$. It also implies that their belief in $\omega$ is always positive upon seeing $\rho_+$.

Agents will deviate from $\rho_+$ into $\rho_\ell$ if and only if there exists an agent whose belief in $\omega_0$ is stronger than its belief in $\omega_i$ and there exists a different agent whose belief in $\omega_0$ is stronger than its belief in  $\omega_j$, for $\{i, j\} = e_\ell$. Together with the above observation, it implies that all agents have equal, positive posteriors for $\omega_0, \omega_i$ or for $\omega_0, \omega_j$.

Agents will deviate from $\rho_+$ into $\rho'_{|V| + 1}$ if and only if there exists an agent that the sum of its belief over $\omega_i, 1 \le i \le |V|$ is greater than $k$ times its belief in $\omega_0$.

Together, all the above observations imply that a policy is stable if and only if the posterior of every agent, upon seeing a signal involved with $\rho_+$, is forming a graph cover of size $k$ defined by all vertices $v_i$ such that the posterior of $\omega_i$ is equal to the posterior of $\omega_0$.
\end{proof}

\subsection{Proofs in \Cref{sc:semi-private}}

The proof of \Cref{lmm:blocking-conditions} uses Hall's theorem \cite{hall1934representation}, which can be stated as follows.


\lmmhallvariant*

\begin{proof}
Given a collection of of $m$ sets
$B_1, \dots, B_m \subseteq \{1, \dots, \ell\}$, the Hall's theorem \cite{hall1934representation} can be stated as follows:
there exist $m$ distinct numbers $b_1, \dots, b_m$, $b_i \in B_i$ for all $i=1,\dots,m$, if and only if it holds for every set $M \subseteq \{1,\dots,m\}$ that $|\bigcup_{i \in M} B_i| \ge |M|$.

To see that the generalization stated in the lemma holds, we make $r_i$ copies of each $B_i$ and examine the existence of a distinct number for each of the $\sum_{i =1}^m r_i$ copies. The ``only if'' direction of the statement then follows immediately. To see that the ``if'' direction also holds, observe that the union of a collection of sets in the new instance does not change if we remove some copies of a set $B_i$ from the collection, as long as one copy of $B_i$ remains in the collection.
\end{proof}

\lmmblockingconditions*

\begin{proof}

Suppose that $s = (\actionvector, \signalvector)$ is stable, and consider an arbitrary deviation $\delta \in \Deviationprofiles_{\rho_\actionvector}$.
We show that we can find a tuple $(A', N')$ that satisfied the stated conditions.

Pick an arbitrary subtype $(T,a)$ such that $\sum_{a' \in A} \delta(T, a, a') > 0$ (which must exist as otherwise $\sum_{T, a, a'} \delta(T, a, a') = 0$, contradicting \Cref{eq:D-rho}). For every $a' \in A^+ := \{ a' \in A: \delta(T, a, a') > 0\}$, we define the following:
\begin{align*}
B_{a'} &= \left \{ i \in T: a_i = a, u_T(a, \rho_\actionvector \given p_i) < u_T(a', \rho_\actionvector \oplus \delta \given p_i) \right\}; \\
r_{a'} &= \delta(T, a, a').
\end{align*}
Namely, $B_{a'}$ consists of the set of agents who are of subtype $(T,a)$ and are willing to participate in $\delta$ if their role is to play $a'$ in the deviation.

Since $s$ is stable, by definition, there does not exist a collection of disjoint sets $\widetilde{B}_{a'} \subseteq B_{a'}$, $a' \in A^+$, such that $\widetilde{B}_{a'} = r_{a'}$.
According to the generalization of Hall's theorem in \Cref{lmm:hall-variant}, this means that there exists $A' \subseteq A^+$ such that
\[
\left|\bigcup\nolimits_{a' \in A'} B_{a'} \right| < \sum\nolimits_{a' \in A'} r_{a'} = \sum\nolimits_{a' \in A'} \delta(T, a, a').
\]
Letting $N' = \{i \in T: a_i = a\} \setminus \bigcup\nolimits_{a' \in A'} B_{a'}$ then gives 
\[
\rho_\actionvector(T,a) - |N'|
< \sum\nolimits_{a' \in A'} r_{a'} = \sum\nolimits_{a' \in A'} \delta(T, a, a').
\]
Moreover, by definition, $i \notin B_{a'}$ for every $i \in N'$ and $a' \in A'$, so \Cref{eq:sm-blocking-profile-uT} holds for every $i \in N'$ and $a' \in A'$.
As a result, all the four conditions stated in the lemma holds for $(A', N')$.

Hence, for every $\delta \in \Deviationprofiles_{\rho_\actionvector}$, there exists $(A', N')$ that satisfies all the conditions. By letting $\gamma(\delta) = (A',N')$, we obtain the desired function.
\end{proof}

To prove \Cref{thm:rp-semi-private}, we use the following lemma.

\begin{restatable}{lemma}{lmmRpSemiPrivateRepresentative}
\label{lmm:rp-semi-private-representative}
For any stable semi-private policy $\sigma : \Omega \to \Delta(\Actionvectors \times G)$, there exists a stable semi-private policy $\tilde{\sigma} : \Omega \to \Delta(\Representativeactionvectors \times \widetilde{G})$ that yields as much utility for the principal as $\sigma$ does.
\end{restatable}

\begin{proof}
We convert $\sigma$ into $\tilde{\sigma}$ as stated in the lemma as follows.
For each meta-signal $s = (\actionvector, \signalvector)$ in the support of $\sigma$, we create a meta-signal $\tilde{s} = (\bar{\actionvector}, \tilde{\signalvector})$ for $\tilde{\sigma}$; whenever $\sigma$ sends $s$, we let $\tilde{\sigma}$ sends $\tilde{s}$.

Specifically, the private part $\tilde{\signalvector}$ of $\tilde{s}$ is defined as follows.
Note that $\bar{\actionvector}$ defines a permutation $\pi$ of $N$, such that in $\bar{\actionvector}$ every agent $i \in N$ will be performing the action performed by agent $\pi(i)$ in $\actionvector$. To define $\tilde{\signalvector}$, we let $\tilde{g}_i = (\actionvector, g_{\pi(i)})$ for each $i \in N$. This way we encode the original action recommendation $\actionvector$ into the private part.
Effectively, the trick makes the information to agent $i$ as informative as that to agent $\pi(i)$ in the original policy. Indeed, the posterior agent $i$ derives when $s'$ is sent by $\sigma'$ is exactly the same as the posterior derived by agent $\pi(i)$ when $s$ is sent by $\sigma$.

Formally, note that by the above construction each agent $i$ will receive meta-signals of the form $(\bar{\actionvector}, (\actionvector, x))$.
For every possible $x$, we have 
\begin{align*}
\tilde{\sigma}_i \left[ \left(\bar{\actionvector}, \left(\actionvector, x \right) \right) \given \omega \right]
&=
\sigma_{\pi(i)} \left[ \left( {\actionvector}, x \right) \given \omega \right],
\end{align*}
where we denote by $\tilde{\sigma}_i$ the signal distribution to agent $i$ (and similarly for $\sigma_{\pi(i)}$). 
It follows that  
\begin{align*}
\widetilde{\Pro}\left[ \omega \given \tilde{s}_i = \left(\bar{\actionvector}, \left(\actionvector, x \right) \right) \right] 
&= 
\frac{\mu(\omega) \cdot \tilde{\sigma}_i \left[ \left(\bar{\actionvector}, \left(\actionvector, x \right) \right) \given \omega \right] }{\widetilde{\Pro}\left[ \tilde{s}_i = \left(\bar{\actionvector}, \left(\actionvector, x \right) \right) \right]} \\
&= 
\frac{\mu(\omega) \cdot \sigma_{\pi(i)} \left[ \left(\actionvector, x \right) \given \omega \right] }{\Pro\left[ s_{\pi(i)} =\left(\actionvector, x \right) \right]} \\
&= 
\Pro\left[ \omega \given s_{\pi(i)} = \left(\actionvector, x \right) \right], 
\end{align*}
where $\widetilde{\Pro}$ and $\Pro$ denote the probability measures induced by $\tilde{\sigma}$ and $\sigma$, respectively.

This means that agent $i$ would derive the same posterior as agent $\pi(i)$ does under the original policy.
Since in $\bar{\actionvector}$ agent $i$ is also recommended to perform the same action agent $\pi(i)$ performs in $\actionvector$, it follows that the new meta-signals must also be stable as are the original ones. 
It is then easy to see that $\tilde{\sigma}$ induce the same distribution over action profiles. Hence, overall, the principal obtains the same expected utility in both cases. 
\end{proof}

\thmRpSemiPrivate*

\begin{proof}
By \Cref{lmm:rp-semi-private-representative}, we can first convert $\sigma$ to a policy described in the statement of the lemma. Hence, we assume that $\sigma$ has the form $\Omega \to \Delta( \Representativeactionvectors \times G)$.

We will next merge the meta-signals in $\sigma$ as we did in the proof of \Cref{thm:rp-public}, with the difference that we now merge meta-signals with the same semi-private signature.
Given $\sigma$, we construct $\bar{\sigma}$: for every $\bar{\actionvector} \in \Representativeactionvectors$ and $\beta \in \smblockingprofiles_{\bar{\actionvector}}$, we let 
    \begin{equation}
    \label{eq:construct-sigma-prime-semi-private}
    \bar{\sigma}( (\bar{\actionvector}, \beta) \given \omega) = \sum_{s: \sigsm(s) = (\bar{\actionvector},\beta)} \sigma( s \given \omega).
    \end{equation}
Namely, whenever the original policy generates a meta-signal $s = (\bar{\actionvector}, \signalvector)$, the new policy generates a meta-signal $(\bar{\actionvector}, \beta)$ that is exactly the signature of $s$.
Each component $(\bar{\actionvector}, \beta_i)$ of the meta-signal is then sent to agent $i$.
Note that due to the property implied by \Cref{lmm:rp-semi-private-representative}, the operation will not change the joint action in each meta-signal.
From \Cref{eq:construct-sigma-prime-semi-private}, we also get that 
    \begin{align}
    \bar{\sigma}( (\bar{\actionvector}, \beta_i) \given \omega) 
    &=
    \sum_{s : \sigsm_i(s) = (\bar{\actionvector}, \beta_i)} \sigma( s \given \omega) \nonumber
    \\
    &=
    \sum_{s_i : \sigsm(s_i) = (\bar{\actionvector}, \beta_i)} \sigma( s_i \given \omega).
    \label{eq:rp-sm-agent-i}
    \end{align}
    Namely, the marginal probability of each component $(\bar{a}_i, \beta_i)$ under the new policy is the same as the marginal probability under the original policy of components $s_i$ whose signatures are $(\bar{a}_i, \beta_i)$.

    We verify the following to complete the proof: 1) $\bar{\sigma}$ is a valid policy, 2) $\bar{\sigma}$ is stable, and 3) $\bar{\sigma}$ yields the same utility for the principal.

    \paragraph{Validity}
    First, $\bar{\sigma}$ is valid because different signatures correspond to disjoint sets of meta-signals, which means that $\sum_{({\actionvector}, \beta) \in \bar{S}} \bar{\sigma}( ({\actionvector}, \beta) \given \omega) = 1$.

    \paragraph{Stability}
    Consider an arbitrary meta-signal $c = (\bar{\actionvector}, \beta)$ in the support of $\bar{\sigma}$.
    We first argue that after we merge the meta-signals, the posterior induced by each $c_i = (\bar{\actionvector}, \beta_i)$ is a convex combination of the posteriors of the merged meta-signals.
    Specifically, let $\bar{\Pro}$ denote the probability measure induced by $\bar{\sigma}$; for every world $\omega$, we have
    \begin{align*}
    \bar{\Pro}(\omega \given c_i)
    &= 
    \frac{\mu(\omega) \cdot \bar{\sigma}(c_i \given \omega)}{ \bar{\Pro}(c_i)} \\
    &= 
    \frac{\mu(\omega) \sum_{s_i: \sigsm(s_i) = c_i} 
    \sigma(s_i \given \omega)}{ \bar{\Pro}(c_i)} 
    & \text{(by \Cref{eq:rp-sm-agent-i})}\\
    &=
    \sum_{s_i: \sigsm(s_i) = c_i}
    \frac{\Pro(s_i)}{\bar{\Pro}(c_i)} \cdot 
    \frac{\mu(\omega) \cdot \sigma(s_i \given \omega)}{\Pro(s_i)} \\
    &=
    \sum_{s_i: \sigsm(s_i) = c_i}
    \frac{\Pro(s_i)}{\bar{\Pro}(c_i)} \cdot 
    \Pro(\omega \given s_i).
    \end{align*}
    The coefficients satisfy
    $\sum_{s_i: \sigsm(s_i) = c_i}
    \frac{\Pro(s_i)}{\bar{\Pro}(c_i)}  = 1$ given the construction in \Cref{eq:construct-sigma-prime-semi-private}. Hence, $\bar{\Pro}( \cdot \given c_i)$ is a convex combination of $\mathcal{P} = \{ \Pro( \cdot \given s_i) : \phi(s_i) = c_i \}$.

    This implies that the signature of $c$ (as a meta-signal of $\bar{\sigma}$) is exactly $(\bar{\actionvector}, \beta)$. 
    Indeed, by definition, we have
    \[
    (\actionvector', i) 
    \in \beta
    \ \Longleftrightarrow\ 
    a_i \neq a'_i \wedge
    u_i(\bar{\actionvector} \given p) \ge u_i(\actionvector' \given p)
    \]
    for all $p \in \mathcal{P}$.
    The utility functions on the two sides of the inequality are linear in $p$.
    So, the inequality holds for $\bar{\Pro}(\omega \given c_i)$ if and only if it holds for all points in $\mathcal{P}$.
    
    By assumption $\sigma$ is stable, so every $s$ in the support of $\sigma$---particularly those with signature $(\bar{\actionvector},\beta)$---is stable.
    This implies that $\beta \in \smblockingprofiles_{\bar{\actionvector}}$. Hence, $c$ is stable.

    \paragraph{Utility Equality}
    Since $\bar{\sigma}$ is stable, the agents are incentivized to perform the recommended actions.
    By construction, the original policy $\sigma$ and the new policy $\bar{\sigma}$ always recommend the same actions. Hence, the utilities these two policies yield are the same.
\end{proof}

\thmOptSmPolyTime*

\begin{proof}
The key is to analyze the size of $\smblockingprofiles_{\bar{\actionvector}}$, or equivalently, the number of possible $\gamma$ functions.
Similarly to our analysis in the proof of \Cref{thm:opt-pub-poly-time}, $D_{\rho_{\actionvector}}$ has a constant size.
So the questions is how many possible $(A',N')$ tuples can we assign to each $\delta \in D_{\rho_{\actionvector}}$?
We know there are constantly many ways of choosing $A'$ since $A$ is constant by assumption. 
Additionally, given the third condition in \Cref{def:sm-blocking-profile}, the size of $N'$ is at least $\rho_\actionvector(T,a) - \sum_{a' \in A'} \delta(T, a, a') \ge \rho_\actionvector(T,a) - d$.
Hence, the choice of $N'$ corresponds to first choosing a subtype $(T,a)$, and then choosing a subset of at least $\rho_\actionvector(T,a) - d$ agents out of the $\rho_\actionvector(T,a)$ agents who are of subtype $(T,a)$.
The ways to choose the latter is polynomial in $n$ when $d$ is constant. This completes the analysis.
\end{proof}



\subsection{Proofs in \Cref{sc:private}}

\thmPrvRepresentativeNotOpt*

\begin{proof}
Consider the following example.
Suppose that there are only one world and one agent type (so, in what follows, we omit the world and the type in the notations).
There are two agents indexed $1$ and $2$, and there are two actions available to each of them, in the set $\Actions = \{a_1, a_2\}$. We have $d=1$, so the agents cannot coordinate their deviation.

Hence, there are three possible action profiles, $\rho_1, \rho_2, \rho_3$, where $\rho_1(a_1) = 2$, $\rho_2(a_1) = \rho_2(a_2) = 1$, and $\rho_3(a_2) = 2$.
The agents' utilities under each profile are given as follows:
\begin{align*}
    & u_T(a_1, \rho_1) = u_T(a_2, \rho_3) = 0 \\
    & u_T(a_1, \rho_2) = u_T(a_2, \rho_2) = 1 
\end{align*}
The utility of the principal is $1$ for $\rho_1$ and $0$ for every other profile.

It can be verified that the policy which recommends the action vectors $(a_1, a_2)$ and $(a_2, a_1)$ with probability $\frac{3}{8}$ each and $(a_1, a_1)$ with probability $\frac{2}{8}$ is stable. Specifically, upon receiving $a_2$, the agents can infer that the profile must be $\rho_2$, which gives each of them the maximum possible utility; and upon receiving $a_1$, the agents can infer that there is a higher probability that the profile is $\rho_2$ than it is $\rho_1$, so they still prefer to not deviate as they would end in $\rho_3$ when the profile is $\rho_2$. Overall, the expected utility the policy yields for the principal is $\frac{2}{8}$ (which is equal to the probability of recommending $\rho_1$).

Indeed, the action vectors $(a_1, a_2)$ and $(a_2, a_1)$ in the above policy both correspond to the same action profile $\rho_2$. 
In the public and semi-private cases, it suffices to use only one of them (as the representative of $\rho_2$) to construct an optimal policy.
In private persuasion, however, this is insufficient. 
Specifically, every policy that does not use $(a_1, a_2)$ 
(i.e., it uses $(a_2, a_1)$ as a representative of $\rho_2$) 
and recommends $(a_1, a_1)$ with a positive probability cannot stable since agent $1$ is always better off changing its action to $a_2$ upon receiving $a_1$ (since he can be sure that the profile recommended is $\rho_1$). This means that no private policy that uses only representative action vectors can achieve a positive utility; hence, they cannot be optimal given the policy that yields utility $\frac{2}{8}$ we presented above.
\end{proof}

\lmmblockingconditionsprivate*

\begin{proof}
The lemma can be proved the same way as \Cref{lmm:blocking-conditions}.
\end{proof}

\thmRpPrivate*

\begin{lemma}
\label{lmm:tilde-lottery-stable}
    For every stable policy $\sigma$, the policy $\Tilde{\lotterypolicy}(\sigma)$ is stable and yields the same expected utility for the principal. 
    Moreover, the posteriors of agent $i$ about the action profile and the world upon receiving $(\actionvector_{m(i)}, (g_{m(i)}, m(i)))$ under $\Tilde{\lotterypolicy}(\sigma)$ are identical to that of agent $m(i)$ upon receiving $(\actionvector_{m(i)}, g_{m(i)})$ under $\sigma$.
\end{lemma}

\begin{proof}
In the private case agents have posteriors over both the world and the actions of other agents.

For a signal $s = (\actionvector', (g', m))$ of a policy $\Tilde{\lotterypolicy}(\sigma)$ we denote $s^{-1} = (\actionvector, g)$ the sign it was derived from under $\sigma$, for which $\actionvector'_{m(i)} = \actionvector_i, g'_{m(i)} = g_i$ for every agent $i$.

We can see that, for every signal $(\actionvector', (g', m)) = s' \in \supp(\Tilde{\lotterypolicy}(\sigma))$ derived from $(\actionvector, g) = s \in \supp(\sigma)$, the posterior of $(\actionvector'', \omega)$ of agent $i \in T$ from receiving $s'_i = (\actionvector'_i, (g'_i, m(i))) = (\actionvector_{m(i)}, (g_{m(i)}, m(i)))$ under $\Tilde{\sigma} = \Tilde{\lotterypolicy}(\sigma)$ is $0$ if $\actionvector''_i \ne \actionvector'_i$, otherwise:

\begin{align}
&\Pro(\actionvector'', \omega | (\actionvector'_i, (g'_i, m(i)))) 
\nonumber\\
=&
\frac{\Pro(\actionvector'', \omega)\Pro((\actionvector'_i, (g'_i, m(i)))|\actionvector'', \omega)}{\sum_{\Tilde{\actionvector}', \omega'} \Pro(\Tilde{\actionvector}', \omega')\Pro((\Tilde{\actionvector}'_i, (g'_i, m(i)))|\Tilde{\actionvector}', \omega')}
\nonumber\\
=&
\frac{\mu(\omega) \sum_{\substack{(\actionvector'', (\Tilde{g}', \Tilde{m})) = \Tilde{s}' \\ : \Tilde{g}'_i = g'_i, \Tilde{m}(i)=m(i)}} \Tilde{\sigma}( (\actionvector'', (\Tilde{g}', \Tilde{m})) | \omega)}
{\sum_{\substack{\omega' \in \Omega, (\Tilde{\actionvector}', (\Tilde{g}', \Tilde{m})) = \Tilde{s}' \\ : \Tilde{\actionvector}'_i = \actionvector''_i, \Tilde{g}_i = g'_i, \Tilde{m}(i)=m(i) }} \mu(\omega') \Tilde{\sigma}((\Tilde{\actionvector}', (\Tilde{g}', \Tilde{m})) | \omega')}
\nonumber\\
=&
\frac{\mu(\omega) \sum_{\substack{(\actionvector'', (\Tilde{g}', \Tilde{m})) = \Tilde{s}' \\ : \Tilde{g}'_i = g'_i, \Tilde{m}(i)=m(i)}} \Pro(\Tilde{m}) \sigma((\Tilde{s}')^{-1} | \omega)}
{\sum_{\substack{\omega' \in \Omega, (\Tilde{\actionvector}', (\Tilde{g}', \Tilde{m})) = \Tilde{s}' \\ : \Tilde{\actionvector}'_i = \actionvector''_i, \Tilde{g}_i = g'_i, \Tilde{m}(i)=m(i) }}  \mu(\omega') \Pro(\Tilde{m}) \sigma((\Tilde{s}')^{-1} | \omega')}
\nonumber\\
=&
\frac{\mu(\omega) \sum_{\substack{(\actionvector'', (\Tilde{g}', \Tilde{m})) = \Tilde{s}' \\ : \Tilde{g}'_i = g'_i, \Tilde{m}(i)=m(i)}} \sigma((\Tilde{s}')^{-1} | \omega)}
{\sum_{\substack{\omega' \in \Omega, (\Tilde{\actionvector}', (\Tilde{g}', \Tilde{m})) = \Tilde{s}' \\ : \Tilde{\actionvector}'_i = \actionvector''_i, \Tilde{g}_i = g'_i, \Tilde{m}(i)=m(i) }} \mu(\omega') \sigma((\Tilde{s}')^{-1} | \omega')}
\end{align}
According to the definition of $\Tilde{\sigma}$, when $\Tilde{m}(i) = m(i)$ then $\Tilde{g}'_i = g'_i$ iff $\Tilde{g}_{m(i)}=g_{m(i)}$. Similarly, $\Tilde{\actionvector}'_i = \actionvector''_i$ iff $\Tilde{\actionvector}_{m(i)} = \actionvector_{m(i)}$. We continue with the above equation:
\begin{align}
=
\frac{\mu(\omega) \sum_{\substack{(\actionvector'', (\Tilde{g}', \Tilde{m})) = \Tilde{s}' \\ : \Tilde{g}_{m(i)} = g_{m(i)}, \Tilde{m}(i)=m(i)}} \sigma((\Tilde{s}')^{-1} | \omega)}
{\sum_{\substack{\omega' \in \Omega, (\Tilde{\actionvector}', (\Tilde{g}', \Tilde{m})) = \Tilde{s}' \\ : \Tilde{\actionvector}_{m(i)} = \actionvector_{m(i)}, \Tilde{g}_{m(i)} = g_{m(i)}, \Tilde{m}(i)=m(i) }} \mu(\omega') \sigma((\Tilde{s}')^{-1} | \omega')}
\end{align}
And, since every derived signal is determined a permutation:
\begin{align}
&=
\frac{\mu(\omega) \sum_{\substack{\Tilde{m} \in M, \Tilde{s} = (\Tilde{\actionvector}, \Tilde{g}) \\ : \Tilde{g}_{m(i)} = g_{m(i)}, \Tilde{m}(i)=m(i) \\ \forall j: \actionvector''_j = \Tilde{\actionvector}_{\Tilde{m}(j)}}} \sigma((\Tilde{\actionvector}, \Tilde{g}) | \omega)}
{\sum_{\substack{\omega' \in \Omega, \Tilde{m} \in M, \Tilde{s} = (\Tilde{\actionvector}, \Tilde{g}) \\ : \Tilde{\actionvector}_{m(i)} = \actionvector_{m(i)}, \Tilde{g}_{m(i)} = g_{m(i)}, \Tilde{m}(i)=m(i)}} \mu(\omega') \sigma((\Tilde{\actionvector}, \Tilde{g}) | \omega')}
\nonumber\\
&=
\frac{\mu(\omega) \sum_{\substack{\Tilde{s} = (\Tilde{\actionvector}, \Tilde{g}) \\ : \Tilde{g}_{m(i)} = g_{m(i)}, \Tilde{\actionvector}_{m(i)} = \actionvector_{m(i)}}} \sum_{\substack{\Tilde{m} \in M \\ : \Tilde{m}(i)=m(i) \\ \forall j: \actionvector''_j=\Tilde{\actionvector}_{\Tilde{m}(j)}}} \sigma((\Tilde{\actionvector}, \Tilde{g}) | \omega)}
{\sum_{\substack{\omega' \in \Omega, \Tilde{s} = (\Tilde{\actionvector}, \Tilde{g}) \\ : \Tilde{\actionvector}_{m(i)} = \actionvector_{m(i)}, \Tilde{g}_{m(i)} = g_{m(i)}}} \sum_{\substack{\Tilde{m} \in M \\ : \Tilde{m}(i)=m(i)}} \mu(\omega') \sigma((\Tilde{\actionvector}, \Tilde{g}) | \omega')}
\end{align}

If $\rho_{\Tilde{\actionvector}}= \rho_{\actionvector''}$ and $\actionvector''_i = \Tilde{\actionvector}_{m(i)}$, the number of permutations $\Tilde{m}: \Tilde{m}(i) = m(i)$ from $\Tilde{\actionvector}$ to $\actionvector''$ depends only on $\rho_{\Tilde{\actionvector}}= \rho_{\actionvector''}$ and is denoted by $c_{\rho_{\actionvector''}}$. We also denote $d$ the number of permutations $\Tilde{m}: \Tilde{m}(i) = m(i)$. We continue:

\begin{align}
&=
\frac{\mu(\omega) \sum_{\substack{\Tilde{s} = (\Tilde{\actionvector}, \Tilde{g}) \\ : \Tilde{g}_{m(i)} = g_{m(i)}, \Tilde{\actionvector}_{m(i)} = \actionvector_{m(i)}}} \sigma((\Tilde{\actionvector}, \Tilde{g}) | \omega) \sum_{\substack{\Tilde{m} \in M \\ : \Tilde{m}(i)=m(i) \\ \forall j: \actionvector''j=\Tilde{\actionvector}_{\Tilde{m}(i)}}} 1}
{\sum_{\substack{\omega' \in \Omega, \Tilde{s} = (\Tilde{\actionvector}, \Tilde{g}) \\ : \Tilde{\actionvector}_{m(i)} = \actionvector_{m(i)}, \Tilde{g}_{m(i)} = g_{m(i)}}} \mu(\omega') \sigma((\Tilde{\actionvector}, \Tilde{g}) | \omega') \sum_{\substack{\Tilde{m} \in M \\ : \Tilde{m}(i)=m(i)}}1}
\nonumber\\
&=
\frac{\mu(\omega) \sum_{\substack{\Tilde{s} = (\Tilde{\actionvector}, \Tilde{g}) \\ : \Tilde{g}_{m(i)} = g_{m(i)}, \Tilde{\actionvector}_{m(i)} = \actionvector_{m(i)}}} \sigma((\Tilde{\actionvector}, \Tilde{g}) | \omega) c_{\rho_{\actionvector''}}}
{\sum_{\substack{\omega' \in \Omega, \Tilde{s} = (\Tilde{\actionvector}, \Tilde{g}) \\ : \Tilde{\actionvector}_{m(i)} = \actionvector_{m(i)}, \Tilde{g}_{m(i)} = g_{m(i)}}} \mu(\omega') \sigma((\Tilde{\actionvector}, \Tilde{g}) | \omega') d}
\nonumber\\
&=
\frac{c_{\rho_{\actionvector''}}}
{d}
\frac{\mu(\omega) \sum_{\substack{\Tilde{s} = (\Tilde{\actionvector}, \Tilde{g}) \\ : \Tilde{g}_{m(i)} = g_{m(i)}, \Tilde{\actionvector}_{m(i)} = \actionvector_{m(i)}}} \sigma((\Tilde{\actionvector}, \Tilde{g}) | \omega)}
{\sum_{\substack{\omega' \in \Omega, \Tilde{s} = (\Tilde{\actionvector}, \Tilde{g}) \\ : \Tilde{\actionvector}_{m(i)} = \actionvector_{m(i)}, \Tilde{g}_{m(i)} = g_{m(i)}}} \mu(\omega') \sigma((\Tilde{\actionvector}, \Tilde{g}) | \omega')}
\nonumber\\
&=
\frac{c_{\rho_{\actionvector''}}}
{d}
\Pro(\rho_{\actionvector''}, \omega | (\actionvector_{m(i)}, g_{m(i)}))
\end{align}
Hence,
\begin{align}
&\Pro(\rho'', \omega | (\actionvector'_i, (g'_i, m(i)))) \nonumber\\ 
=& \sum_{\substack{\actionvector'': \actionvector''_i = \actionvector'_i \\ \rho_{\actionvector''} = \rho''}} \Pro(\actionvector'', \omega | (\actionvector'_i, (g'_i, m(i)))) 
\nonumber \\
=&
\sum_{\substack{\actionvector'': \actionvector''_i = \actionvector'_i \\ \rho_{\actionvector''} = \rho''}} \frac{c_{\rho_{\actionvector''}}}{d}
\Pro(\rho_{\actionvector''}, \omega | (\actionvector_{m(i)}, g_{m(i)}))
\nonumber \\
=&
\sum_{\substack{\actionvector'': \actionvector''_i = \actionvector'_i \\ \rho_{\actionvector''} = \rho''}} \frac{c_{\rho''}}{d}
\Pro(\rho'', \omega | (\actionvector_{m(i)}, g_{m(i)}))
\nonumber \\
=&
\Pro(\rho'', \omega | (\actionvector_{m(i)}, g_{m(i)})) \sum_{\substack{\actionvector'': \actionvector''_i = \actionvector'_i \\ \rho_{\actionvector''} = \rho''}} \frac{c_{\rho''}}{d} 1
\end{align}
That is, the posteriors of $\rho'', \omega$ upon getting $\actionvector'_i, (g'_i, m(i))$ under $\Tilde{\sigma}$ are proportionate to the posteriors upon getting $\actionvector_{m(i)}, g_{m(i)}$ under $\sigma$. Since these posteriors sum to $1$, we know that these proportion is actually $1$, hence:

\begin{align}
\Pro(\rho'', \omega | (\actionvector'_i, (g'_i, m(i))))
=
\Pro(\rho'', \omega | (\actionvector_{m(i)}, g_{m(i)}))
\end{align}

Since these posteriors are identical, the expected utility of every agent $i$ under $\Tilde{\sigma}$ from every deviation is identical to that of $m(i)$ under $\sigma$, and since $\sigma$ is stable, $\Tilde{\sigma}$ must be stable too.
\end{proof}

\begin{lemma}
\label{lmm:private-representative-vectors}
     There exists a private policy $\sigma: \Worlds \rightarrow \Delta(\Representativeactionvectors \times G)$, for some $G$, such that $\Tilde{\lotterypolicy}(\sigma)$ is an optimal private policy.
\end{lemma}

\begin{proof}
    Let $\sigma'$ be an optimal policy. We know that $\Tilde{\lotterypolicy}(\sigma')$ is also an optimal policy. We define $\sigma$ a policy that signals $(\bar{\actionvector}, \signalvector')$ when $\sigma$ signals $(\actionvector, \signalvector)$, where $\bar{\actionvector}$ is the representative of $\rho_\actionvector$ and $g'_i=g_{f(i)}$, where $f$ is the permutation that turns $\actionvector$ to its representative.

    It's easy to see that $\sigma$ results in the same utility as $\sigma'$, since the profiles of the signaled vectors are always identical. In addition, since the vectors and private signals of $\sigma$ are mere permutations of these of $\sigma'$, the policies $\Tilde{\lotterypolicy}(\sigma), \Tilde{\lotterypolicy}(\sigma')$ are identical, and so $\Tilde{\lotterypolicy}(\sigma)$ is an optimal policy that signals only representative action vectors.
\end{proof}

\begin{lemma}
\label{lmm:before-after-lottery}
    For any private policy $\sigma$, the policy $b(\Tilde{\lotterypolicy}(\sigma))$ is identical to the policy $\lotterypolicy(b(\sigma))$, where $b(\sigma)$ denotes the policy resulting from merging signals with the same private signature in $\sigma$, by signaling the private blocking profile of signals instead. 
    Moreover, if $\sigma$ is stable, then $b(\sigma)$ is stable and yields the same utility for the principal.
\end{lemma}

\begin{proof}
    The first part of the statement is a direct result of \Cref{lmm:tilde-lottery-stable}. Specifically, since the posteriors over profiles are identical to that of the original agent the private part of the signal were derived from, the blocking profile, and hence the private signature, must be identical, too.

    The second part, about the stability and utility of $\sigma$, follows by a similar argument as that of \Cref{thm:rp-semi-private}, only that now we merge private signals.
\end{proof}



\begin{proof}[Proof of \Cref{thm:rp-private}]
By \Cref{lmm:private-representative-vectors} we know there exists a policy $\sigma: \Worlds \rightarrow \Delta(\Representativeactionvectors \times G)$ such that $\Tilde{\lotterypolicy}(\sigma)$ is a stable optimal policy. Now, according to the second part of \Cref{lmm:before-after-lottery}, we can see that $b(\Tilde{\lotterypolicy}(\sigma))$ is also optimal and stable, and by \Cref{lmm:before-after-lottery} the policy $\lotterypolicy(b(\sigma))$ is identical to it, hence stable and optimal too. By definition, the policy $b(\sigma)$ sends only representative action vectors and blocking profiles, i.e., is a function $\Omega \to \Delta(\{(\bar{\actionvector}, \beta): \bar{\actionvector} \in \Representativeactionvectors, \beta \in \pvblockingprofiles\})$.
\end{proof}

\thmOptPrPolyTime*

\begin{proof}
    Similarly to the proof of \Cref{thm:opt-sm-poly-time}, we can see that $|\pvblockingprofiles|$ is polynomial, hence the number of variables and constraints in the LP of the private case are polynomial too. It follows that finding an optimal policy can be done in polynomial time.
\end{proof}



\end{document}